\documentclass{article}
\pdfoutput=1

    \PassOptionsToPackage{numbers, compress}{natbib}

    \usepackage[final]{neurips_2022}

\usepackage[utf8]{inputenc} 
\usepackage[T1]{fontenc}    
\usepackage{hyperref}       
\usepackage{url}            
\usepackage{booktabs}       
\usepackage{amsfonts}       
\usepackage{nicefrac}       
\usepackage{microtype}      
\usepackage{xcolor}         
\usepackage{graphicx}
 \graphicspath{{figures/}} 
\usepackage{float}
\newfloat{figtab}{htb}{fgtb}
\makeatletter
  \newcommand\figcaption{\def\@captype{figure}\caption}
  \newcommand\tabcaption{\def\@captype{table}\caption}

\usepackage{amsmath}
\usepackage{amssymb}
\usepackage{amsthm}
\usepackage{subfigure}
\usepackage{wrapfig}        
\usepackage{algorithm}
\usepackage{algorithmic}

\newtheorem{proposition}{Proposition}
\newtheorem{definition}{Definition}

\newtheorem{observation}{Observation}

\hypersetup{
	colorlinks=true
}

\title{Mildly Conservative $Q$-Learning for Offline Reinforcement Learning}

\author{%
  Jiafei Lyu$^{1}$\thanks{Equal Contribution}, Xiaoteng Ma$^{2*}$, Xiu Li$^{1}$\thanks{Corresponding Authors}, Zongqing Lu$^{3 \dagger}$ \\
  $^{1}$Tsinghua Shenzhen International Graduate School, Tsinghua University\\
  $^{2}$Department of Automation, Tsinghua Unversity \\
  $^{3}$School of Computer Science, Peking University \\
  \texttt{\{lvjf20,ma-xt17\}@mails.tsinghua.edu.cn,} \\
  \texttt{li.xiu@sz.tsinghua.edu.cn, zongqing.lu@pku.edu.cn}
}

\begin{document}

\maketitle

\begin{abstract}

Offline reinforcement learning (RL) defines the task of learning from a static logged dataset without continually interacting with the environment. The distribution shift between the learned policy and the behavior policy makes it necessary for the value function to stay conservative such that out-of-distribution (OOD) actions will not be severely overestimated. However, existing approaches, penalizing the unseen actions or regularizing with the behavior policy, are too pessimistic, which suppresses the generalization of the value function and hinders the performance improvement. This paper explores mild but enough conservatism for offline learning while not harming generalization. We propose Mildly Conservative $Q$-learning (MCQ), where OOD actions are actively trained by assigning them proper pseudo $Q$ values. We theoretically show that MCQ induces a policy that behaves at least as well as the behavior policy and no erroneous overestimation will occur for OOD actions. Experimental results on the D4RL benchmarks demonstrate that MCQ achieves remarkable performance compared with prior work. Furthermore, MCQ shows superior generalization ability when transferring from offline to online, and significantly outperforms baselines. Our code is publicly available at \href{https://github.com/dmksjfl/MCQ}{https://github.com/dmksjfl/MCQ}.

\end{abstract}

\section{Introduction}

Continually interacting with the environment of online reinforcement learning (RL) is often infeasible and unrealistic, since the data collection process of the agent may be expensive, difficult, or even dangerous, especially in real-world applications. Offline RL, instead, aims at learning from a static dataset that was previously collected by some unknown process \cite{Lange2012BatchRL}, hence eliminating the need for environmental interactions during training. 

The main challenge of offline RL is the distribution shift of state-action visitation frequency between the learned policy and the behavior policy. The evaluation of out-of-distribution (OOD) actions causes extrapolation error \cite{Fujimoto2019OffPolicyDR}, which can be exacerbated through bootstrapping \cite{Kumar2019StabilizingOQ} and result in severe overestimation errors. Thus, keeping conservatism in value estimation is necessary in offline RL~\cite{jin2021pessimism,rashidinejad2021bridging,xie2021bellman}. Previous methods achieve the conservatism by compelling the learned policy to be close to the behavior policy \cite{Fujimoto2019OffPolicyDR, Wu2019BehaviorRO, Kumar2019StabilizingOQ, Fujimoto2021AMA, Wang2020CriticRR}, by penalizing the learned value functions from being over-optimistic upon out-of-distribution (OOD) actions \cite{Kumar2020ConservativeQF, Kostrikov2021OfflineRL, Wu2021UncertaintyWA}, or by learning without querying OOD samples \cite{Wang2018ExponentiallyWI, Chen2020BAILBI,yang2021believe,kostrikov2022offline,ma2022offline}. 


In practice, we rely on neural networks to extract knowledge from the dataset and generalize it to the nearby unseen states and actions when facing continuous state and action spaces. In other words, we need the networks to ``stitch'' the suboptimal trajectories to generate the best possible trajectory supported by the dataset. Unfortunately, there is no free lunch. Conservatism, which offline RL celebrates, often limits the generalization and impedes the performance of the agent. Existing approaches are still inadequate in balancing conservatism and generalization. As illustrated in Figure~\ref{fig:mcb}, policy regularization is unreliable for offline RL when the data-collecting policy is poor, and value penalization methods often induce unnecessary pessimism in both the in-dataset region and OOD region. We argue that \emph{the proper conservatism should be as mild as possible}. As depicted in Figure~\ref{fig:mcb}, we aim at well estimating the value function in the support of the dataset, and allowing value estimates upon OOD actions to be high (even higher than their optimal values) as long as $Q(s,a^{\rm ood})< \max_{a\in{\rm Support}(\mu)}Q(s,a)$ is satisfied. The mild conservatism benefits generalization since value estimates upon OOD actions are slightly optimistic instead of being overly conservative.


\begin{figure}
    \centering
    \includegraphics[width=\textwidth]{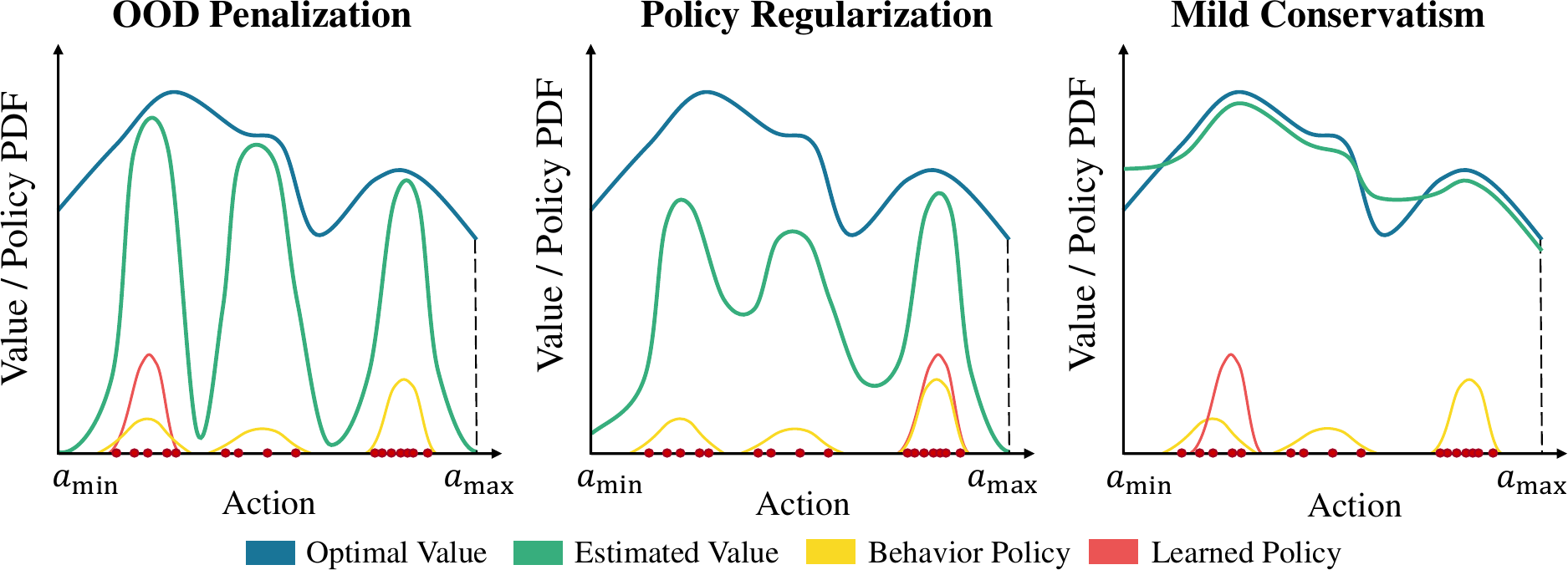}
    \caption{Comparison of prior methods against mild conservatism. The red spots represent the dataset samples. The left figure shows that penalizing OOD actions makes the value function drop sharply at the boundary of the dataset's support, which barriers policy learning. The central figure depicts that policy regularization keeps the policy near behavior policy, leading to undesired performance if the behavior policy is unsatisfying. On the right side, we illustrate the basic idea of mild conservatism. The estimated values for OOD actions are allowed to be high as long as it does not affect the learning for the optimal policy supported by the dataset, i.e., $Q(s,a^{\rm ood})< \max_{a\in{\rm Support}(\mu)}Q(s,a)$.}
    \label{fig:mcb}
\end{figure}



To fulfill that, we propose a novel \emph{\underline{M}ildly \underline{C}onservative \underline{B}ellman} (MCB) operator for offline RL, where we \emph{actively} train OOD actions and query their $Q$ values. We theoretically analyze the convergence property of the MCB operator under the tabular MDP setting. We show that the policy induced by the MCB operator is guaranteed to behave better than the behavior policy, and can consistently improve the policy with a tighter lower bound compared with policy constraint methods or value penalization methods like CQL \cite{Kumar2020ConservativeQF}. For practical usage, we propose the practical MCB operator and illustrate its advantages by theoretically showing that \emph{erroneous overestimation error will not occur} with it. We then estimate the behavior policy with a conditional variational autoencoder (CVAE) \cite{Kingma2014AutoEncodingVB,Sohn2015LearningSO}, and integrate the practical MCB operator with the Soft Actor-Critic (SAC) \cite{Haarnoja2018SoftAO} algorithm. To this end, we propose our novel offline RL algorithm, Mildly Conservative $Q$-learning (MCQ).

Experimental results on the D4RL MuJoCo locomotion tasks demonstrate that MCQ surpasses recent strong baseline methods on most of the tasks, especially on non-expert datasets. Meanwhile, MCQ shows superior generalization capability when transferring from offline to online, validating our claims that mild pessimism is of importance to offline learning.

\section{Preliminaries}
We consider a Markov Decision Process (MDP) specified by a tuple $\langle \mathcal{S},\mathcal{A},r,\rho_0,p,\gamma \rangle$, where $\mathcal{S}$ is the state space, $\mathcal{A}$ is the action space, $r(s,a):\mathcal{S}\times\mathcal{A}\mapsto\mathbb{R}$ is the reward function, $\rho_0(s)$ is the initial state distribution, $p(s^\prime|s,a):\mathcal{S}\times\mathcal{A}\times\mathcal{S}\mapsto[0,1]$ is the transition probability, $\gamma\in[0,1)$ is the discount factor. Reinforcement learning (RL) aims at finding a policy $\pi(\cdot|s)$ such that the expected cumulative long-term rewards $J(\pi) = \mathbb{E}_{s_0\sim\rho_0(\cdot), a_t\sim\pi(\cdot|s_t), s_{t+1}\sim p(\cdot|s_t,a_t)}[\sum_{t=0}^\infty\gamma^t r(s_t,a_t)]$ are maximized. The state-action function $Q(s,a)$ measures the discounted return starting from state $s$ and action $a$, and following the policy $\pi$. We assume that the reward function $r(s,a)$ is bounded, i.e., $|r(s,a)|\le r_{\rm max}$. Given a policy $\pi(\cdot|s)$, the Bellman backup for obtaining the corresponding $Q$ function gives:
\begin{equation}
    \label{eq:bellmanbackup}
    \mathcal{T}^\pi Q(s,a):= r(s,a) + \gamma\mathbb{E}_{s^\prime}\mathbb{E}_{a^\prime\sim\pi(\cdot|s^\prime)}[Q(s^\prime, a^\prime)].
\end{equation}
The $Q$ function of the optimal policy satisfies the following Bellman optimal operator:
\begin{equation}
    \label{eq:bellmanoptimal}
    \mathcal{T}Q(s,a) := r(s,a) + \gamma\mathbb{E}_{s^\prime}\left[\max_{a^\prime\in\mathcal{A}}Q(s^\prime, a^\prime)\right].
\end{equation}
In offline RL setting, the online interaction is infeasible, and we can only have access to previously collected datasets $\mathcal{D}=\{(s_i,a_i,r_i,s_{i+1}^\prime,d_i)\}_{i=1}^N$, where $d$ is the done flag. We denote the behavior policy as $\mu(\cdot|s)$. The Bellman backup relies on actions sampled from the learned policy, $a^\prime\sim\pi(\cdot|s^\prime)$. However, $a^\prime$ can lie outside of the support of $\mu$ due to the distribution shift between $\pi$ and $\mu$. The value estimates upon $a^\prime$ can then be arbitrarily wrong, resulting in bad policy training. Unlike prior work, we \emph{actively} train OOD actions by constructing them pseudo target values. In this way, we retain pessimism while enjoying better generalization.

\section{Mildly Conservative $Q$-Learning}

In this section, we first formally define the MCB operator and characterize its dynamic programming properties in the tabular MDP setting. We further give a practical version of the MCB operator. We show that no erroneous overestimation will occur with the MCB operator. Finally, we incorporate the MCB operator with SAC \cite{Haarnoja2018SoftAO} and present our novel offline RL algorithm.

\subsection{Mildly Conservative Bellman (MCB) Operator}

\begin{definition}
\label{def:mcb}
The Mildly Conservative Bellman (MCB) operator is defined as 
\begin{equation}
    \mathcal{T}_{\mathrm{MCB}} Q(s,a) = (\mathcal{T}_1\mathcal{T}_2) Q(s,a),
\end{equation}
where
\begin{equation}
\label{eq:t1}
    \mathcal{T}_1 Q(s,a) = \left\{  
         \begin{array}{lr}  
         Q(s,a), & \mu(a|s) > 0. \\  
        \max_{a^\prime\sim{\rm Support}(\mu(\cdot|s))}Q(s,a^\prime) - \delta, & \text{else}.\\
         \end{array}  
\right.  
\end{equation}
\begin{equation}
\label{eq:t2}
    \mathcal{T}_2 Q(s,a) = \left\{  
         \begin{array}{lr}  
         r(s,a) + \gamma \mathbb{E}_{s^\prime} \left[\max_{a^\prime\in\mathcal{A}} Q(s^\prime,a^\prime)\right],  & \mu(a|s)>0,  \\  
         {\color{blue} r(s,a) + \gamma \mathbb{E}_{s^\prime} \left[ Q(s,a) \right]}, & \text{else}. \\
         \end{array}  
\right.  
\end{equation}
\end{definition}

{\color{blue} \textbf{Update 2024.02:} The authors were informed that there is a bug in Equation \ref{eq:t2} and made the above modifications to fix it. The authors would like to thank Jiayi Guan for pinpointing this. The above modification has marginal influence on theoretical results and no effect on practical MCQ algorithm.}


The basic idea behind this novel operator is that if the learned policy outputs actions that lie in the support region of $\mu$, then we go for backup; while if OOD actions are generated, we deliberately replace their value estimates with $\max_{a^\prime\sim{\rm Support}(\mu(\cdot|s))}Q(s,a^\prime) - \delta$, where $\delta>0$ can be arbitrarily small. That is, different from standard Bellman backup, we set up a \emph{checking procedure} (i.e., $\mathcal{T}_1$) of whether the previous backup (i.e., $\mathcal{T}_2$) involves OOD actions for the update. Intrinsically, we construct pseudo target values for OOD actions. We subtract a small positive $\delta$ such that OOD actions will not be chosen when executing policy via $\arg\max_{a\in\mathcal{A}} Q(s,a)$. 


For a better understanding of the MCB operator, we theoretically analyze its dynamic programming properties in the tabular MDP setting. All proofs are deferred to Appendix \ref{sec:proof}.

\begin{proposition}
\label{prop:contraction}
In the support region of the behavior policy, i.e., ${\rm Support}(\mu)$, the MCB operator is a $\gamma$-contraction operator in the $\mathcal{L}_\infty$ norm, and any initial $Q$ function can converge to a unique fixed point by repeatedly applying $\mathcal{T}_{\mathrm{MCB}}$.
\end{proposition}

\begin{proposition}[Behave at least as well as behavior policy]
\label{prop:behaveatleastaswellas}
Denote $Q_{\mathrm{MCB}}$ as the unique fixed point acquired by the MCB operator, then in ${\rm Support}(\mu)$ we have: $Q_\mu \le Q_{\mathrm{MCB}}\le Q_{\mu^*}$, where $Q_\mu$ is the $Q$ function of the behavior policy and $Q_{\mu^*}$ is the $Q$ function of the optimal policy in the batch.
\end{proposition}


Proposition \ref{prop:behaveatleastaswellas} indicates that the policy induced by the MCB operator can behave at least as well as the behavior policy, and can approximate the optimal batch-constraint policy. Apart from this advantage, we further show that the MCB operator results in milder conservatism. We start by observing that value penalization method, like CQL \cite{Kumar2020ConservativeQF}, guarantees that the learned value function $\hat{Q}^\pi(s,a)$ is a lower bound of its true value $Q^\pi(s,a)$. It is also ensured that following such conservative update leads to a safe policy improvement, i.e., $J(\pi_{\rm CQL})\ge J(\mu) - \mathcal{O}(\frac{1}{(1-\gamma)^2})$ (Theorem 3.6 in \cite{Kumar2020ConservativeQF}). For explicit policy constraint methods, e.g., TD3+BC \cite{Fujimoto2021AMA}, the learned policy $\pi_p$ mimics the behavior policy $\mu$, and can hardly behave significantly better than $\mu$. We show in Proposition \ref{prop:milderpessimism} that explicit policy constraint methods also exhibit a safe policy improvement, $J(\pi_p)\ge J(\mu) - \mathcal{O}(\frac{1}{(1-\gamma)^2})$, while the MCB operator can consistently improve the policy with a tighter lower bound.


\begin{proposition}[Milder Pessimism]
\label{prop:milderpessimism}
Suppose there exists an explicit policy constraint offline reinforcement learning algorithm such that the KL-divergence of the learned policy $\pi_p(\cdot|s)$ and the behavior policy $\mu(\cdot|s)$ is optimized to guarantee $\max\left(\mathrm{KL}(\mu, \pi_p), \mathrm{KL}(\pi_p, \mu) \right)\le\epsilon,\,\forall s$. Denote $\epsilon_\mu^{\pi_p} = \max_s|\mathbb{E}_{a\sim\pi_p}A^\mu(s,a)|$, where $A^\mu(s,a)$ is the advantage function. Then
\begin{equation}
    J(\pi_p) \ge J(\mu) - \dfrac{\sqrt{2} \gamma\epsilon_\mu^{\pi_p}}{(1-\gamma)^2}\sqrt{\epsilon},
\end{equation}
while for the policy $\pi_{\rm MCB}$ learned by applying the MCB operator, we have
\begin{equation}
    J(\pi_{\rm MCB}) \ge J(\mu).
\end{equation}
\end{proposition}

\textbf{In summary,} the MCB operator benefits the offline learning in two aspects: (1) the operator is a contraction, and any initial $Q$ functions are guaranteed to converge to a unique fixed point; (2) the learned policy of the MCB operator is ensured to be better than the behavior policy, and reserve milder pessimism compared with policy constraint methods or CQL.

\subsection{Practical MCB Operator}

In practice, it is intractable to acquire $\max_{a^\prime\sim{\rm Support}(\mu(\cdot|s))}Q(s,a^\prime)$ in $\mathcal{T}_1$ of Eq. (\ref{eq:t1}) in continuous control domains, and the behavior policy is often unknown. Thus, we fit an empirical behavior policy $\hat{\mu}$ with supervised learning based on the static dataset. The pseudo target values for the OOD actions are then computed by sampling $N$ actions from $\hat{\mu}$, and taking maximum over their value evaluation. Formally, we define the practical MCB operator below, accompanied by the theoretical analysis.

\begin{definition}
The practical Mildly Conservative Bellman (MCB) operator is defined as 
\begin{equation}
    \mathcal{\hat{T}}_{\mathrm{MCB}} Q(s,a) = (\mathcal{\hat{T}}_1\mathcal{T}_2) Q(s,a),
\end{equation}
where
\begin{equation}
\label{eq:practicalmcb}
    \mathcal{\hat{T}}_1 Q(s,a) = \left\{  
         \begin{array}{lr}  
         Q(s,a), & \mu(a|s) > 0. \\  
        \mathbb{E}_{\{a_i^\prime\}^N\sim\hat{\mu}(\cdot|s)}\left[\max_{a^\prime\sim\{a_i^\prime\}^N}Q(s,a^\prime)\right], & \text{else}.\\
         \end{array}  
\right.  
\end{equation}
\end{definition}
Compared with Eq. (\ref{eq:t1}), we make a small modification of $\mathcal{T}_1$, and keep $\mathcal{T}_2$ unchanged. There is no need to subtract $\delta$ here as generally $\mathbb{E}_{\{a_i^\prime\}^N\sim\hat{\mu}(\cdot|s)}\left[\max_{a^\prime\sim\{a_i^\prime\}^N}Q(s,a^\prime)\right]\le\max_{a^\prime\sim{\rm Support}(\mu)}Q(s,a^\prime)$. The practical MCB operator is much easier to implement in practice. We show that the practical MCB operator is still a $\gamma$-contraction in the support region of the behavior policy $\mu$.

\begin{proposition}
\label{pro:1}
Proposition \ref{prop:contraction} still holds for the practical MCB operator.
\end{proposition}


Since we fit the empirical distribution $\hat{\mu}$ of the behavior policy $\mu$, there may exist a shift between $\hat{\mu}$ and $\mu$, especially when we represent the policy via neural networks. That suggests that OOD actions $a^\prime$ can still be sampled from $\hat{\mu}$ such that $a^\prime\notin {\rm Support}(\mu(\cdot|s))$. Our last main result reveals that \emph{erroneous overestimation issue will not occur} with the aid of the practical MCB operator.


\begin{proposition}[No erroneous overestimation will occur]
\label{pro:2}
Assuming that $\sup_s D_{\rm TV}( \hat \mu( \cdot|s) \parallel \mu( \cdot|s)) \leq \epsilon < \frac{1}{2}$, we have
\begin{align*}
    \mathbb{E}_{\{a_i^\prime\}^N\sim \hat{\mu}(\cdot|s)} \left[\max_{a^\prime\in\{a_i^\prime\}^N}Q(s,a^\prime)\right] \leq \max_{a^\prime \in  \operatorname{Support}(\mu(\cdot|s))} Q(s,a^\prime) + (1 - (1 - 2\epsilon)^N) \dfrac{r_{\rm max}}{1-\gamma}.
\end{align*}
\end{proposition}

\noindent\textbf{Remark:} This proposition generally requires a comparatively well-fitted empirical behavior policy $\hat{\mu}$. In practice, we model $\hat{\mu}$ with a CVAE. In most cases, CVAE can already fit the dataset well and guarantee a good performance. Whereas there may exist some situations, e.g., the dataset is highly multi-modal, then one can replace the CVAE as the conditional GAN (CGAN) to better capture the different modes in the dataset as depicted in \cite{Yang2022ARI}. We believe generative models like CGAN will be a good choice by then.


Intuitively, the above conclusion says that if the empirical behavior policy $\hat{\mu}$ well fits $\mu$, i.e., $\epsilon$ is small enough, then regardless of how $\{a_i^\prime\}^N$ are sampled, the pseudo target value will approximate the maximum $Q$-value within the dataset's support with high probability. The extrapolation error is under the scale of $(1 - (1-2\epsilon)^N) \frac{r_{\rm max}}{1-\gamma}$. We expect a good empirical behavior policy such that most of the actions sampled from it will be in-distribution. However, if $\epsilon$ is large, $N$ can act as a trade-off parameter. The smaller $N$ we use, the more conservative we are. Fortunately, we find empirically that our method performs well in a large interval of $N$ over different tasks (see Section~\ref{sec:parameterstudy}). Hence, it is safe to fix a $N$ in practice.

\subsection{Algorithm}
As aforementioned, we often cannot get prior information about the behavior policy $\mu$. Thus, we need to empirically fit a behavior policy $\hat{\mu}$ with supervised learning for applying the practical MCB operator. Our algorithm, Mildly Conservative $Q$-learning (MCQ), trains an additional generative model, which is also adopted by many prior work \cite{Fujimoto2019OffPolicyDR, Ghasemipour2021EMaQEQ, Kostrikov2021OfflineRL, Zhou2020PLASLA}. We build our novel offline algorithm upon an off-the-shelf off-policy online RL algorithm, Soft Actor-Critic (SAC) \cite{Haarnoja2018SoftAO}.

\noindent\textbf{Modelling the behavior policy with the CVAE.} We utilize a conditional variational autoencoder (CVAE) \cite{Kingma2014AutoEncodingVB,Sohn2015LearningSO,Fujimoto2019OffPolicyDR} to model the behavior policy $\mu$. Given a fixed logged dataset, the goal of the CVAE is to reconstruct actions conditioned on the states such that the reconstructed actions come from the same distribution as the actions in the dataset, i.e., $\mu(\cdot|s)$. That generally satisfies the assumption we make in Proposition \ref{pro:2}. As concerned by \cite{kostrikov2022offline}, training a generative model like CVAE still may produce out-of-dataset actions, which leads to extrapolation error since undefined $Q$ values can be possibly queried. Prior methods, like BCQ \cite{Fujimoto2019OffPolicyDR}, do not well address such issue. While for our algorithm, such concern is mitigated because overestimation error is actually under control as is guaranteed by Proposition \ref{pro:2}.

The CVAE $G_\omega(s)$ parameterized by $\omega$ is made up of an encoder $E_\xi(s,a)$ and a decoder $D_{\psi}(s,z)$ parameterized by $\xi,\psi$ respectively, $\omega=\{\xi,\psi\}$. The CVAE is optimized by maximizing its variational lower bound, which is equivalent to minimizing the following objective function.
\begin{equation}
    \label{eq:CVAE}
    \mathcal{L}_{\rm CVAE} = \mathbb{E}_{(s,a)\sim\mathcal{D},z\sim E_\xi(s,a)}\left[ (a - D_\psi(s,z))^2 + {\rm KL}\left(E_\xi(s,a), \mathcal{N}(0,{\bf I})\right) \right],
\end{equation}
where ${\rm KL}(p,q)$ denotes the KL-divergence between probability distribution $p(\cdot)$ and $q(\cdot)$, and $\bf I$ is the identity matrix. When sampling actions from the CVAE, we first sample a latent variable $z$ from the prior distribution, which is set to be multivariate normal distribution $\mathcal{N}(0, \bf I)$, and then pass it in conjunction with the state $s$ into the decoder $D_\psi(s,z)$ to get the desired decoded action.

It is also worth noting that we do not choose GAN \cite{Goodfellow2014GenerativeAN} as the generative model because it is known to suffer from training instability and mode collapse \cite{Srivastava2017VEEGANRM,Bau2019SeeingWA,Bang2021MGGANSM}. Also, GAN consumes much more time and memories to train compared with the CVAE. 


\noindent\textbf{Loss functions.} In deep RL, the $Q$ function is represented with a neural network parameterized by $\theta$ and is updated via minimizing the temporal difference (TD) loss $\mathbb{E}_{s,a,r,s^\prime}[(Q_\theta(s,a) - \mathcal{T}Q(s,a))^2]$. We actually are performing the regression task $(s,a)\mapsto\mathcal{T}Q(s,a)$ to train the $Q$ function. The target value $\mathcal{T}Q(s,a)$ is usually computed by utilizing a lagging target network parameterized by $\theta^\prime$ without gradient backpropagation. As a typical actor-critic \cite{Konda1999ActorCriticT,Konda1999ActorCriticA,sutton2018reinforcement} algorithm, SAC uses its critic networks to perform value estimation and uses a separate actor network for policy improvement. In order to incorporate the MCB operator with the off-the-shelf SAC algorithm, we need to check whether the sampled action $a^\prime \sim\pi(\cdot|s)$ lies outside of the behavior policy's support, i.e., whether ${\mu}(a^\prime|s)>0$. However, such a criterion is not reliable, because the true behavior policy $\mu$ is unknown and it is difficult to examine whether $\mu(a^\prime|s)>0$ in practice. It is also problematic if we rely on the empirical behavior policy $\hat{\mu}$ to check whether $a^\prime$ is OOD as $\hat{\mu}$ itself can produce OOD actions.

\begin{algorithm}[tb]
\caption{Mildly Conservative $Q$-learning (MCQ)}
\label{alg:algmcq}
\begin{algorithmic}[1] 
\STATE Initialize CVAE $G_\omega$, critic networks $Q_{\theta_1}, Q_{\theta_2}$ and actor network $\pi_{\phi}$ with random parameters 
\STATE Initialize target networks $\theta_1^\prime \leftarrow \theta_1, \theta_2^\prime \leftarrow \theta_2$ and offline replay buffer $\mathcal{D}$.
\FOR{$t$ = 1 to $T$}
\STATE Sample a mini-batch $B = \{(s,a,r,s^\prime,d)\}$ from $\mathcal{D}$, where $d$ is the done flag
\STATE Train CVAE via minimizing Eq. (\ref{eq:CVAE})
\STATE Get target value: $y = r(s,a) + \gamma\left[\min_{i=1,2}Q_{\theta_i^\prime}(s^\prime,a^\prime) - \alpha \log\pi_\phi(a^\prime|s^\prime)\right], a^\prime\sim\pi_\phi(\cdot|s^\prime)$
\STATE Sample $N$ actions from $\pi$ based on each $s$ and $s^\prime$, set $s^{\rm in} = \{s,s^\prime\}$
\STATE Compute the target value for the OOD actions via Eq. (\ref{eq:faketarget})
\STATE Update critic $\theta_i$ with gradient descent via minimizing Eq. (\ref{eq:criticloss})
\STATE Update actor $\phi$ with gradient ascent via Eq. (\ref{eq:actorloss})
\STATE Update target networks: $\theta_i^\prime\rightarrow\tau\theta_i + (1-\tau)\theta_i^\prime,\, i=1,2$
\ENDFOR
\end{algorithmic}
\end{algorithm}

We then resort to constructing an auxiliary loss for OOD actions and integrating it with the standard Bellman error. Specifically, we sample $a^{\rm ood}$ from the learned policy $\pi(\cdot|s^{\rm in})$ based on the sampled state $s^{\rm in}\sim\mathcal{D}$ from the dataset and assign them pseudo target values based on the practical MCB operator. Note that the superscript $\rm ood$ is used to distinguish from the in-dataset real actions, and $a^{\rm ood}$ is not necessarily an OOD action. We remark that if $a^{\rm ood}\in\mathrm{Support}(\mu(\cdot|s))$, the pseudo $Q$ value will not negatively affect the evaluation upon it, because in-distribution actions are still trained to approximate the optimal batch-constraint $Q$ value. In this way, we \emph{actively} train both possible OOD actions and in-distribution actions simultaneously via \emph{OOD sampling}. The resulting objective function for the critic networks is presented in Eq. (\ref{eq:criticloss}).
\begin{equation}
\label{eq:criticloss}
    \mathcal{L}_{\rm critic} = \lambda \mathbb{E}_{(s,a,r,s^\prime)\sim \mathcal{D}}\left[ (Q_{\theta_i}(s,a) - y)^2 \right] + (1-\lambda){\color{blue} \mathbb{E}_{s^{\rm in}\sim\mathcal{D}, a^{\rm ood}\sim\pi}\left[ (Q_{\theta_i}(s^{\rm in}, a^{\rm ood}) - y^\prime)^2 \right]},
\end{equation}
where the target value for the in-distribution actions gives
\begin{equation}
    \label{eq:truetarget}
    y = r(s,a) + \gamma\left[\min_{i=1,2}Q_{\theta_i^\prime}(s^\prime,a^\prime) - \alpha \log\pi_\phi(a^\prime|s^\prime)\right], \alpha\in\mathbb{R}_+,
\end{equation}
which follows the standard target value of vanilla SAC. The hyperparameter $\lambda$ balances the in-distribution data training and OOD action training. Following the formulas of the practical MCB operator in Eq. (\ref{eq:practicalmcb}), the pseudo target value for the OOD action is computed by:
\begin{equation}
    \label{eq:faketarget}
    y^\prime = \min_{j=1,2}\mathbb{E}_{\{a_i'\}^N\sim\hat{\mu}}\left[ \max_{a^\prime\sim\{a_i'\}^N}Q_{\theta_j}(s^{\rm in},a^\prime) \right].
\end{equation}
Note that we experimentally find that replacing the min operator with a mean operator does not raise much difference in performance. We hence take advantage of the min operator to fulfill the pseudo clipped double $Q$-learning for OOD actions.

The policy is then optimized by solving the following optimization problem:
\begin{equation}
\label{eq:actorloss}
    \pi_\phi := \max_\phi \mathbb{E}_{s\sim\mathcal{D},a\sim\pi_\phi(\cdot|s)}\left[ \min_{i=1,2}Q_{\theta_i}(s,a) - \alpha \log\pi_\phi(\cdot|s) \right].
\end{equation}
We detail the learning procedure of our MCQ in Algorithm \ref{alg:algmcq}. Different from \cite{Fujimoto2021AMA}, our method does not require normalization over states or value functions. The only change we make to the vanilla SAC algorithm is an extra auxiliary loss term (blue term in Eq. (\ref{eq:criticloss})) such that OOD actions are actively and properly trained. The additional critic loss term can also be plugged into other off-policy online RL algorithms directly. As an evidence, we combine the MCB operator with TD3 \cite{Fujimoto2018AddressingFA}, yielding a deterministic version of MCQ. Please refer to Appendix \ref{sec:determcq} for more details.

\section{Experiments}
In this section, we first empirically demonstrate the effectiveness and advantages of our proposed MCQ algorithm on D4RL benchmarks \cite{Fu2020D4RLDF}. We then conduct a detailed parameter study to show the hyperparameter sensitivity of MCQ. We also experimentally illustrate that the value estimation of MCQ will not incur severe overestimation and pessimistic value estimates are witnessed in practice. Finally, we show the superior offline-to-online fine-tuning benefits of MCQ on some MuJoCo datasets.

\subsection{Results on MuJoCo Datasets}

We experimentally compare our MCQ against behavior cloning (BC), SAC, and several recent strong baseline methods, CQL \cite{Kumar2020ConservativeQF}, UWAC \cite{Wu2021UncertaintyWA}, TD3+BC \cite{Fujimoto2021AMA}, and IQL \cite{kostrikov2022offline}, on D4RL \cite{Fu2020D4RLDF} benchmarks. We choose these methods as they typically represent different categories of model-free offline RL, i.e., CQL is a value penalization method, TD3+BC involves explicit policy constraint (BC loss), UWAC relies on uncertainty estimation for training, and IQL learns without querying OOD samples. 


We conduct experiments on MuJoCo locomotion tasks, which are made up of five types of datasets (random, medium, medium-replay, medium-expert, and expert), yielding a total of 15 datasets. We use the most recently released "-v2" datasets for performance evaluation. The results of BC and SAC are acquired by using our implemented code. The results of CQL and UWAC are obtained by running their official codes, because the reported scores in their papers are not obtained on MuJoCo "-v2" datasets. We take the results of TD3+BC from its original paper (Table 7 in \cite{Fujimoto2021AMA}). Since the IQL paper does not report its performance on MuJoCo \emph{random} and \emph{expert} datasets, we run IQL using the official codebase on them and take the results on medium, medium-replay, medium-expert datasets from its original paper directly. All methods are run for 1M gradient steps.

\begin{table}[htb]
  \caption{Normalized average score comparison of MCQ against baseline methods on D4RL benchmarks over the final 10 evaluations. 0 corresponds to a random policy and 100 corresponds to an expert policy. The experiments are run on MuJoCo "-v2" datasets over 4 random seeds. r = random, m = medium, m-r = medium-replay, m-e = medium-expert, e = expert. We \textbf{bold} the highest mean.}
  \label{tab:scorepaper}
  \small
  \centering
  \setlength{\tabcolsep}{3.5pt}
  \begin{tabular}{@{}lllllllll@{}}
    \toprule
    Task Name  & BC & SAC & CQL & UWAC & TD3+BC & IQL & MCQ (ours) \\
    \midrule
    halfcheetah-r & 2.2$\pm$0.0 & \textbf{29.7}$\pm$1.4 & 17.5$\pm$1.5 & 2.3$\pm$0.0 & 11.0$\pm$1.1 & 13.1$\pm$1.3 & 28.5$\pm$0.6 \\
    hopper-r & 3.7$\pm$0.6 & 9.9$\pm$1.5 & 7.9$\pm$0.4 & 2.7$\pm$0.3 & 8.5$\pm$0.6 & 7.9$\pm$0.2 & \textbf{31.8}$\pm$0.5 \\
    walker2d-r & 1.3$\pm$0.1 & 0.9$\pm$0.8 & 5.1$\pm$1.3 & 2.0$\pm$0.4 & 1.6$\pm$1.7 & 5.4$\pm$1.2 & \textbf{17.0}$\pm$3.0 \\
    halfcheetah-m & 43.2$\pm$0.6 & 55.2$\pm$27.8 & 47.0$\pm$0.5 & 42.2$\pm$0.4 & 48.3$\pm$0.3 & 47.4$\pm$0.2 & \textbf{64.3}$\pm$0.2 \\
    hopper-m & 54.1$\pm$3.8 & 0.8$\pm$0.0 & 53.0$\pm$28.5 & 50.9$\pm$4.4 & 59.3$\pm$4.2 & 66.2$\pm$5.7 &  \textbf{78.4}$\pm$4.3 \\
    walker2d-m & 70.9$\pm$11.0 & -0.3$\pm$0.2 & 73.3$\pm$17.7 & 75.4$\pm$3.0 & 83.7$\pm$2.1 & 78.3$\pm$8.7 & \textbf{91.0}$\pm$0.4 \\
    halfcheetah-m-r & 37.6$\pm$2.1 & 0.8$\pm$1.0 & 45.5$\pm$0.7 & 35.9$\pm$3.7 & 44.6$\pm$0.5 & 44.2$\pm$1.2 & \textbf{56.8}$\pm$0.6 \\
    hopper-m-r & 16.6$\pm$4.8 & 7.4$\pm$0.5 & 88.7$\pm$12.9 & 25.3$\pm$1.7 & 60.9$\pm$18.8 & 94.7$\pm$8.6 & \textbf{101.6}$\pm$0.8 \\
    walker2d-m-r & 20.3$\pm$9.8 & -0.4$\pm$0.3 & 81.8$\pm$2.7 & 23.6$\pm$6.9 & 81.8$\pm$5.5 & 73.8$\pm$7.1 & \textbf{91.3}$\pm$5.7 \\
    halfcheetah-m-e & 44.0$\pm$1.6 & 28.4$\pm$19.4 & 75.6$\pm$25.7 & 42.7$\pm$0.3 & \textbf{90.7}$\pm$4.3 & 86.7$\pm$5.3 & 87.5$\pm$1.3 \\
    hopper-m-e & 53.9$\pm$4.7 & 0.7$\pm$0.0 & 105.6$\pm$12.9 & 44.9$\pm$8.1 & 98.0$\pm$9.4 & 91.5$\pm$14.3 & \textbf{111.2}$\pm$0.1 \\
    walker2d-m-e & 90.1$\pm$13.2 & 1.9$\pm$3.9 & 107.9$\pm$1.6 & 96.5$\pm$9.1 & 110.1$\pm$0.5 & 109.6$\pm$1.0 & \textbf{114.2}$\pm$0.7 \\
    \midrule
    Average Above& 36.5 & 11.3 & 59.1 & 37.0 & 58.2 & 59.9 & \textbf{72.8} \\
    \midrule
    halfcheetah-e & 91.8$\pm$1.5 & -0.8$\pm$1.8 & 96.3$\pm$1.3 & 92.9$\pm$0.6 & \textbf{96.7}$\pm$1.1 & 95.0$\pm$0.5 & 96.2$\pm$0.4 \\
    hopper-e & 107.7$\pm$0.7 & 0.7$\pm$0.0 & 96.5$\pm$28.0 & 110.5$\pm$0.5 & 107.8$\pm$7 & 109.4$\pm$0.5 & \textbf{111.4}$\pm$0.4 \\
    walker2d-e & 106.7$\pm$0.2 & 0.7$\pm$0.3 & 108.5$\pm$0.5 & 108.4$\pm$0.4 & \textbf{110.2}$\pm$0.3 & 109.9$\pm$1.2 & 107.2$\pm$1.1 \\
    \midrule
    Total Average & 49.6 & 9.0 & 67.3 & 50.4 & 67.6 & 68.9 & \textbf{79.2} \\
    \bottomrule
  \end{tabular}
\end{table}

In our experiments, we set the number of sampled actions $N=10$ by default and tune the weighting coefficient $\lambda$. We report the $\lambda$ used for all tasks in Appendix \ref{sec:detail}, along with details on the experiments and implementation. We summarize the normalized average score comparison of MCQ against recent baselines in Table \ref{tab:scorepaper}. Unsurprisingly, we observe that MCQ behaves better than BC on all of the tasks, which is consistent with our theoretical analysis in Proposition \ref{prop:behaveatleastaswellas} and \ref{prop:milderpessimism}. MCQ also significantly outperforms the base SAC algorithm. Prior offline RL methods struggle for good performance on non-expert datasets like random and medium-replay, while MCQ surpasses them with a remarkable margin on many non-expert datasets. We attribute the less satisfying performance of prior offline RL methods to their \emph{strict conservatism}, which restricts their generalization beyond the support of the dataset and leads to limited performance. The results, therefore, validate our claim that milder pessimism is more we need for offline learning. Furthermore, MCQ is also competitive to baselines on expert datasets. MCQ achieves the best performance on 11 out of 15 datasets, yielding a total average score of \textbf{72.8} on non-expert datasets, and an average score of \textbf{79.2} on all 15 datasets. Whereas the second best method, IQL, has an average score of 59.9 on non-expert datasets and a total average score of 68.9 across all tasks. 
\begin{figure}
    \centering
    \subfigure[Effects of $\lambda$]{
    \label{fig:halfcheetahlambda}
    \includegraphics[width=0.32\textwidth]{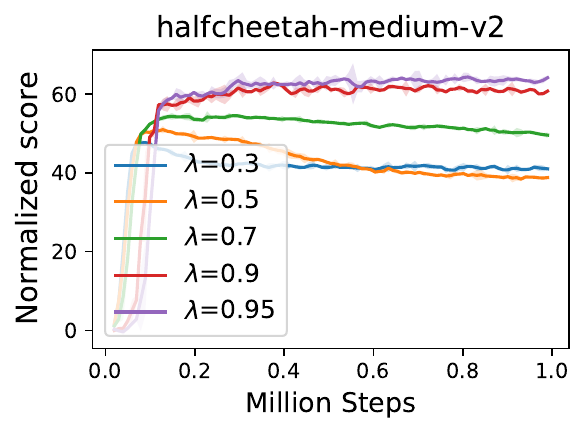}
    }\hspace{-2mm}
    \subfigure[Effects of $\lambda$]{
    \label{fig:hopperlambda}
    \includegraphics[width=0.32\textwidth]{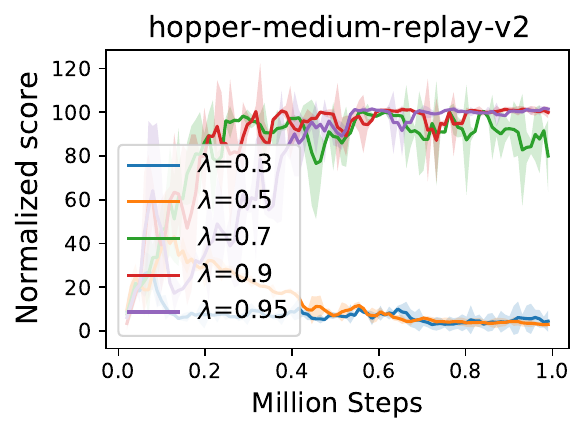}
    }\hspace{-2mm}
    \subfigure[$Q$ value w.r.t. $\lambda$]{
    \label{fig:hopperlambdaq}
    \includegraphics[width=0.32\textwidth]{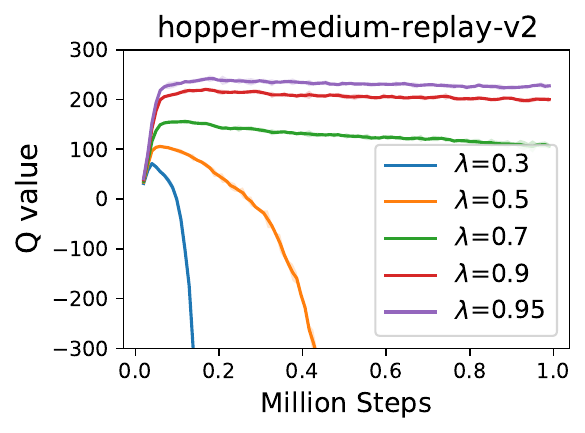}
    }\hspace{-2mm}
    \subfigure[Effects of $N$]{
    \label{fig:halfcheetahN}
    \includegraphics[width=0.32\textwidth]{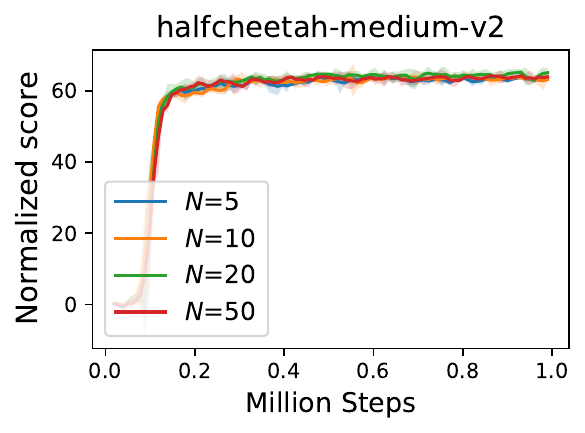}
    }\hspace{-2mm}
    \subfigure[Effects of. $N$]{
    \label{fig:hopperN}
    \includegraphics[width=0.32\textwidth]{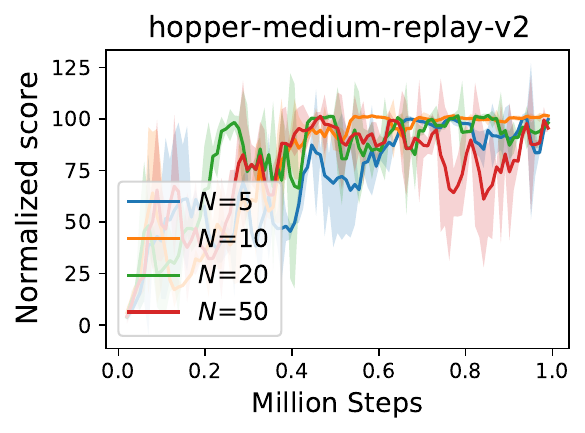}
    }\hspace{-2mm}
    \subfigure[$Q$ value w.r.t. $N$]{
    \label{fig:hopperNq}
    \includegraphics[width=0.32\textwidth]{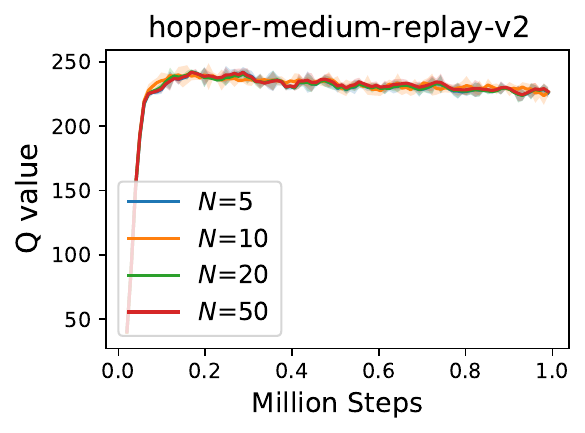}
    }\hspace{-2mm}
    \caption{Parameter study and $Q$ function estimation on halfcheetah-medium-v2 and hopper-medium-replay-v2. The shaded region captures the standard deviation.}
    \label{fig:parameterstudy}
    \vspace{-2mm}
\end{figure}

\subsection{Parameter Study}
\label{sec:parameterstudy}
In this subsection, we conduct a detailed parameter study on MCQ. MCQ generally contains two hyperparameters, weighting coefficient $\lambda$ and number of sampled actions $N$. To demonstrate the parameter sensitivity of MCQ, we choose two datasets from MuJoCo locomotion tasks and conduct experiments on them, halfcheetah-medium-v2, and hopper-medium-replay-v2. The experiments are run for 1M gradient steps over 4 different random seeds.

\textbf{Weighting coefficient $\lambda$.} The weighting coefficient $\lambda$ is a critical hyperparameter for MCQ, which directly controls the balance between in-distribution actions training and OOD actions training. If we set $\lambda=1$, then MCQ degenerates into the base SAC algorithm. If $\lambda$ leans towards 0, the critics will be overwhelmed by OOD actions. Intuitively, one ought not to use small $\lambda$, because more weights are desired for standard Bellman error such that in-distribution state-action pairs can be well-trained. We observe significant performance drop with smaller $\lambda$ in Figure \ref{fig:halfcheetahlambda} and \ref{fig:hopperlambda}. Also, we find that choosing $0.7\le\lambda<1$ generally induces good performance.

\textbf{Number of sampled actions $N$.} $N$ works as a regularizer to control the potential extrapolation error. In case the behavior policy $\mu$ is known, we require $N$ to be as large as possible to better estimate the maximum $Q$ value. While in practice, we leverage the CVAE to approximate $\mu$, from which OOD actions can be sampled. $N$ then plays a role to balance pessimism and generalization. To see the influence of $N$, we fix $\lambda=0.95$ for the two datasets. Experimental results in Figure \ref{fig:halfcheetahN} and \ref{fig:hopperN} indicate that MCQ is insensitive to $N$ for a wide range of $N$. We therefore set $N=10$ by default.

\textbf{Value estimation.} We present the $Q$ value estimates with respect to (w.r.t.) $\lambda$ and $N$ in Figure \ref{fig:hopperlambdaq} and \ref{fig:hopperNq}. The $Q$ estimation is calculated via $\mathbb{E}_{i=1,2}\mathbb{E}_{(s,a)\sim\mathcal{D}}[Q_{\theta_i}(s,a)]$. The results illustrate that (1) smaller $\lambda$ will incur severe underestimation issue (as depicted by Figure \ref{fig:hopperlambdaq}, $Q$ values collapse with $\lambda=0.5$ or $\lambda=0.3$); (2) no overestimation is observed, even with a large $\lambda = 0.95$, which validates the theoretical result in Proposition \ref{pro:2}; (3) the $Q$ estimates resemble each other under different $N$. We conclude that MCQ ensures a stable and good value estimation with a proper $\lambda$.

\subsection{Offline-to-online Fine-tuning}

We examine the offline-to-online fine-tuning capability of MCQ against some prior strong offline RL baselines, CQL \cite{Kumar2020ConservativeQF}, TD3+BC \cite{Fujimoto2021AMA}, IQL \cite{kostrikov2022offline}. We additionally compare against AWAC \cite{Nair2020AcceleratingOR}, which is designed intrinsically for offline-to-online adaptation. We conduct experiments on MuJoCo \emph{random} and \emph{medium-replay} datasets. It is challenging to train on these datasets for both offline and offline-to-online fine-tuning as they are non-expert, or even contain many bad transitions. We first train baselines and MCQ for 1M gradient steps offline and then perform online fine-tuning for another 100K gradient steps. Note that IQL paper \cite{kostrikov2022offline} adopts 1M steps for online fine-tuning. However, we argue that 1M steps of online interactions are even enough to train off-policy online RL algorithms from scratch to perform very well. We thus believe 100K steps is more reasonable for the online interaction. All methods are run over 4 random seeds. The results are shown in Figure \ref{fig:offline2online}, where the shaded region denotes the standard deviation. As expected, we observe that MCQ consistently outperforms prior offline RL methods as well as AWAC on all of the datasets, often surpassing all of them with a large margin. The mild pessimism of MCQ makes it adapt faster, or keep the offline good performance during online interactions. Other prior offline RL methods, unfortunately, fail in achieving satisfying performance during online interaction due to strict conservatism and lack of generalization ability. 

\begin{figure}
    \centering
    \includegraphics[width=\textwidth]{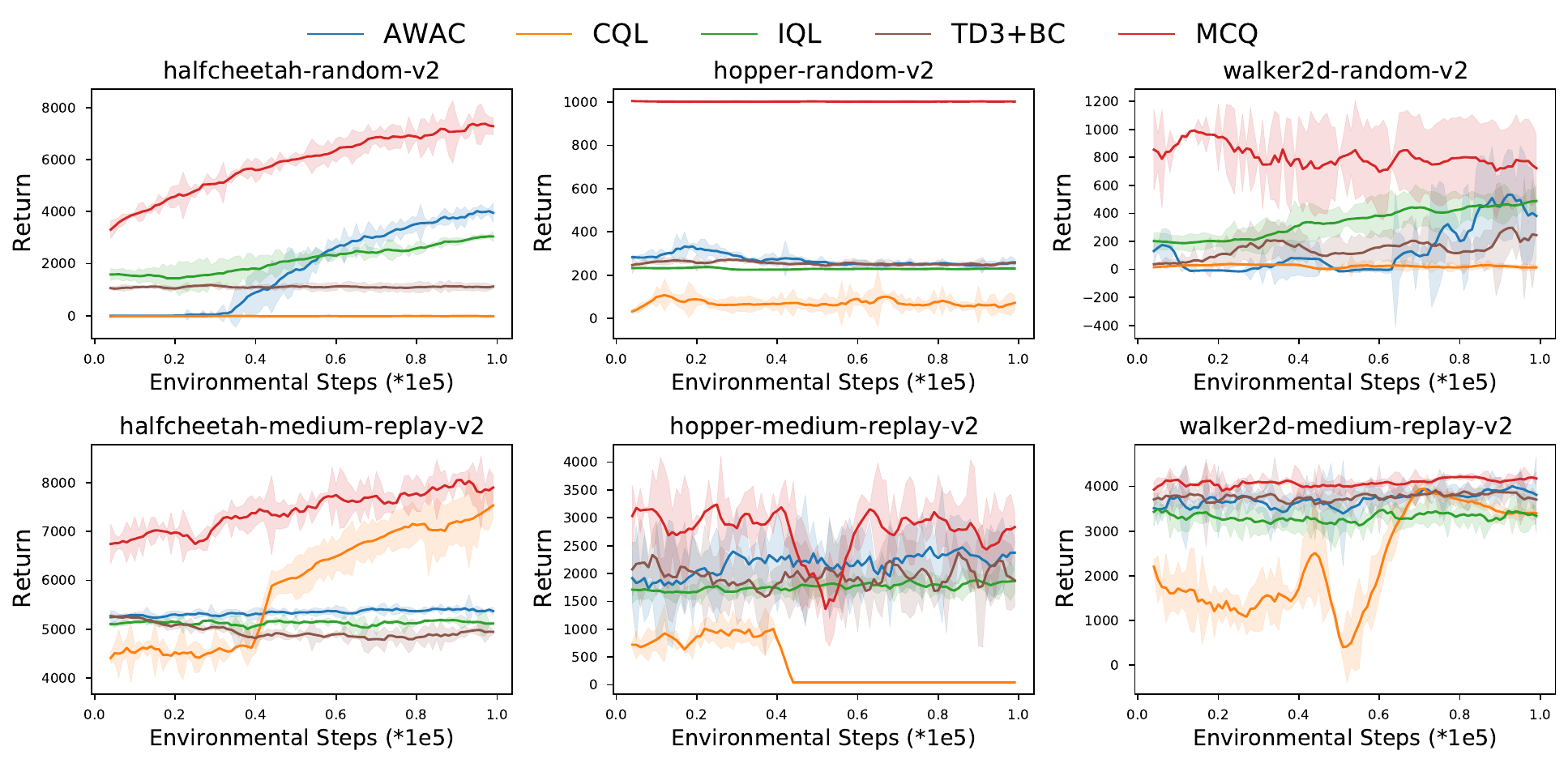}
    \vspace{-4mm}
    \caption{Offline-to-online fine-tuning results on 6 D4RL MuJoCo locomotion tasks.}
    \label{fig:offline2online}
    \vspace{-2mm}
\end{figure}

\section{Related Work}


\noindent \textbf{Model-free offline RL.} Prior model-free offline RL methods are typically designed to restrict the learned policy from producing OOD actions. They usually achieve this by leveraging importance sampling \cite{Precup2001OffPolicyTD, Sutton2016AnEA, Liu2019OffPolicyPG, Nachum2019DualDICEBE, Gelada2019OffPolicyDR}, incorporating explicit policy constraints \cite{Kumar2019StabilizingOQ, Wu2019BehaviorRO, Ghasemipour2021EMaQEQ, Fujimoto2021AMA, Fakoor2021ContinuousDC}, learning latent actions \cite{Zhou2020PLASLA, ajay2021opal}, penalizing learned value functions such that low values are assigned to unseen actions \cite{Kumar2020ConservativeQF, Kostrikov2021OfflineRL, Ma2021ConservativeOD}, using adaptive methods \cite{Ghosh2022OfflineRP}, and uncertainty quantification \cite{Wu2021UncertaintyWA, Zanette2021ProvableBO, bai2022pessimistic}. Another line of the methods, instead, resorts to learning without querying OOD actions \cite{Wang2018ExponentiallyWI, Chen2020BAILBI, kostrikov2022offline}. By doing so, they constrain the learning process within the support of the dataset. Nevertheless, existing methods may induce unnecessarily over-pessimistic value functions, and their performance is largely confined by how well the behavior policy is \cite{Nair2020AcceleratingOR, Lee2021OfflinetoOnlineRL, bai2022pessimistic}. That partly explains why these methods are not satisfiable when trained on non-expert datasets (e.g., random datasets). MCQ keeps milder conservatism and better generalization ability as OOD actions are \emph{actively} trained with proper targets.


\noindent \textbf{Model-based offline RL.} Model-based offline RL methods, in contrast, learn the dynamics model in a supervised manner, and leverage the learned dynamics for policy optimization. Advances in this field include uncertainty quantification \cite{Ovadia2019CanYT, Yu2020MOPOMO, Kidambi2020MOReLM, Diehl2021UMBRELLAUM}, learning conservative value functions \cite{Yu2021COMBOCO}, representation learning \cite{lee2021representation, Rafailov2021OfflineRL}, constraining the learned policy with a behavior cloning loss \cite{matsushima2021deploymentefficient}, and sequential modelling \cite{Chen2021DecisionTR, Janner2021ReinforcementLA, Meng2021OfflinePM}. However, there is no guarantee that the trained dynamics models are reliable, e.g., poor transitions can be generated, especially in complex high-dimensional environments \cite{Janner2019WhenTT}. Meanwhile, training dynamics models raises extra computation costs.


\noindent \textbf{Offline-to-online RL.} There are some efforts on accelerating online interactions with the aid of offline logged data, which is also referred to as learning from demonstration \cite{Hester2018DeepQF, Kang2018PolicyOW, Rajeswaran2018LearningCD}. Offline-to-online RL, instead, aims at enhancing the well-trained offline policy via online interactions. To ensure a fast adaptation and stable policy improvement, many techniques are adopted, such as model ensemble \cite{Lee2021OfflinetoOnlineRL}, explicit policy constraints \cite{Nair2020AcceleratingOR, zhao2022adaptive}. Offline-to-online fine-tuning will be difficult if the trained value function or policy is overly pessimistic, which may lead to a suboptimal policy. 

\section{Conclusion}



In this paper, we propose Mildly Conservative $Q$-learning (MCQ) to alleviate the over pessimism in existing offline RL methods. MCQ actively train OOD actions by constructing them proper pseudo target values following the guidance of the practical Mildly Conservative Bellman (MCB) operator. We theoretically illustrate that the policy induced by the MCB operator behaves at least as well as the behavior policy, and no erroneous overestimation will occur for the practical MCB operator. Furthermore, we extensively compare MCQ against recent strong baselines on MuJoCo locomotion tasks. Experimental results show that MCQ surpasses these baselines with a large margin on many non-expert datasets, and is also competitive with baselines on expert datasets. Moreover, we demonstrate the superior generalization capability of MCQ when transferring from offline to online. These altogether reveal that mild conservatism is critical for offline learning. We hope this work can promote the offline RL towards mild pessimism, and bring new insights into the community.

One drawback of our current algorithm lies in the need of tuning the weighting coefficient $\lambda$. However, we empirically find that $0.7\le\lambda<1$ can usually induce satisfying performance. We leave the automatic tuning of $\lambda$ as future work.

\begin{ack}
This work was supported in part by the Science and Technology Innovation
2030-Key Project under Grant 2021ZD0201404, in part by the NSF China under Grant 61872009. The authors would like to thank the anonymous reviewers for their valuable comments and advice.
\end{ack}

\small
\bibliographystyle{abbrv}
\bibliography{neurips_2022.bib}

\section*{Checklist}

\begin{enumerate}

\item For all authors...
\begin{enumerate}
  \item Do the main claims made in the abstract and introduction accurately reflect the paper's contributions and scope?
    \answerYes{}
  \item Did you describe the limitations of your work?
    \answerYes{}
  \item Did you discuss any potential negative societal impacts of your work?
    \answerNA{}
  \item Have you read the ethics review guidelines and ensured that your paper conforms to them?
    \answerYes{}
\end{enumerate}

\item If you are including theoretical results...
\begin{enumerate}
  \item Did you state the full set of assumptions of all theoretical results?
    \answerYes{}
        \item Did you include complete proofs of all theoretical results?
    \answerYes{}
\end{enumerate}

\item If you ran experiments...
\begin{enumerate}
  \item Did you include the code, data, and instructions needed to reproduce the main experimental results (either in the supplemental material or as a URL)?
    \answerYes{}
  \item Did you specify all the training details (e.g., data splits, hyperparameters, how they were chosen)?
    \answerYes{}
        \item Did you report error bars (e.g., with respect to the random seed after running experiments multiple times)?
    \answerYes{}
        \item Did you include the total amount of compute and the type of resources used (e.g., type of GPUs, internal cluster, or cloud provider)?
    \answerYes{}
\end{enumerate}

\item If you are using existing assets (e.g., code, data, models) or curating/releasing new assets...
\begin{enumerate}
  \item If your work uses existing assets, did you cite the creators?
    \answerYes{}
  \item Did you mention the license of the assets?
    \answerYes{}
  \item Did you include any new assets either in the supplemental material or as a URL?
    \answerNo{}
  \item Did you discuss whether and how consent was obtained from people whose data you're using/curating?
    \answerNo{}
  \item Did you discuss whether the data you are using/curating contains personally identifiable information or offensive content?
    \answerNA{}
\end{enumerate}

\item If you used crowdsourcing or conducted research with human subjects...
\begin{enumerate}
  \item Did you include the full text of instructions given to participants and screenshots, if applicable?
    \answerNA{}
  \item Did you describe any potential participant risks, with links to Institutional Review Board (IRB) approvals, if applicable?
    \answerNA{}
  \item Did you include the estimated hourly wage paid to participants and the total amount spent on participant compensation?
    \answerNA{}
\end{enumerate}

\end{enumerate}

\clearpage

\appendix

\section{Missing Proofs}
\label{sec:proof}
We recall the definition of the MCB operator below.
\begin{definition}
The Mildly Conservative Bellman (MCB) operator is defined as 
\begin{equation}
    \mathcal{T}_{\mathrm{MCB}} Q(s,a) = (\mathcal{T}_1\mathcal{T}_2) Q(s,a),
\end{equation}
where
\begin{equation}
    \mathcal{T}_1 Q(s,a) = \left\{  
         \begin{array}{lr}  
         Q(s,a), & \mu(a|s) > 0. \\  
        \max_{a^\prime\sim{\rm Support}(\mu(\cdot|s))}Q(s,a^\prime) - \delta, & \text{else}.\\
         \end{array}  
\right.  
\end{equation}
\begin{equation}
    \mathcal{T}_2 Q(s,a) = \left\{  
         \begin{array}{lr}  
         r(s,a) + \gamma \mathbb{E}_{s^\prime} \left[\max_{a^\prime\in\mathcal{A}} Q(s^\prime,a^\prime)\right],  & \mu(a|s)>0,  \\  
         {\color{blue} r(s,a) + \gamma \mathbb{E}_{s^\prime} \left[ Q(s,a) \right]}, & \text{else}. \\
         \end{array}  
\right.  
\end{equation}
\end{definition}

\begin{proposition}
\label{prop:prop1}
In the support region of the behavior policy, i.e., ${\rm Support}(\mu)$, the MCB operator is a $\gamma$-contraction operator in the $\mathcal{L}_\infty$ norm, and any initial $Q$ function can converge to a unique fixed point by repeatedly applying $\mathcal{T}_{\mathrm{MCB}}$.
\end{proposition}

\begin{proof}
Let $Q_1$ and $Q_2$ be two arbitrary $Q$ functions. Since $a\in{\rm Support}(\mu(\cdot|s))$, then if $a^\prime\in{\rm Support}(\mu(\cdot|s^\prime))$, we have
\begin{equation*}
\begin{aligned}
    \| \mathcal{T}_{\rm MCB}Q_1 - \mathcal{T}_{\rm MCB}Q_2 \|_\infty &= \| \mathcal{T}_{2}Q_1 - \mathcal{T}_{2}Q_2 \|_\infty \\
    &=\max_{s,a}\left | \left( r(s,a) + \gamma \mathbb{E}_{s^\prime}\left[\max_{a^\prime\in\mathcal{A}}Q_1(s^\prime,a^\prime)\right] \right) - \left( r(s,a) + \gamma \mathbb{E}_{s^\prime}\left[\max_{a^\prime\in\mathcal{A}}Q_2(s^\prime,a^\prime)\right] \right) \right | \\
    &= \gamma \max_{s,a}\left| \mathbb{E}_{s^\prime}\left[ \max_{a^\prime\sim\mathcal{A}}Q_1(s^\prime,a^\prime) - \max_{a^\prime\sim\mathcal{A}}Q_2(s^\prime,a^\prime) \right]  \right| \\
    &\le \gamma \max_{s,a}\mathbb{E}_{s^\prime}\left| \max_{a^\prime\sim\mathcal{A}}Q_1(s^\prime,a^\prime) - \max_{a^\prime\sim\mathcal{A}}Q_2(s^\prime,a^\prime)  \right| \\
    &\le \gamma \max_{s,a} \|Q_1 - Q_2\|_\infty \\
    &= \gamma \|Q_1 - Q_2\|_\infty.
\end{aligned}
\end{equation*}
While if $a^\prime\notin{\rm Support}(\mu(\cdot|s^\prime))$, then $\mathcal{T}_1$ in the MCB operator will work and assign OOD actions pseudo target values. Similarly, if $a^\prime\notin{\rm Support}(\mu(\cdot|s^\prime))$, we have
\begin{equation*}
\begin{aligned}
    &\| \mathcal{T}_{\rm MCB}Q_1 - \mathcal{T}_{\rm MCB}Q_2 \|_\infty \\
    &= \| \mathcal{T}_{1}\mathcal{T}_2 Q_1 - \mathcal{T}_{1}\mathcal{T}_2 Q_2 \|_\infty \\
    &= \max_{s,a}\left | \left( r(s,a) + \gamma \mathbb{E}_{s^\prime}\left[\max_{a^\prime\in{\rm Support}(\mu(\cdot|s^\prime))}Q_1(s^\prime,a^\prime) - \delta\right] \right) - \left( r(s,a) + \gamma \mathbb{E}_{s^\prime}\left[\max_{a^\prime\in{\rm Support}(\mu(\cdot|s^\prime))}Q_2(s^\prime,a^\prime) - \delta\right] \right) \right | \\
    &= \gamma \max_{s,a}\left| \mathbb{E}_{s^\prime}\left[ \max_{a^\prime\in{\rm Support}(\mu(\cdot|s^\prime))}Q_1(s^\prime,a^\prime) - \max_{a^\prime\in{\rm Support}(\mu(\cdot|s^\prime))}Q_2(s^\prime,a^\prime) \right]  \right| \\
    &\le \gamma \max_{s,a}\mathbb{E}_{s^\prime}\left| \max_{a^\prime\in{\rm Support}(\mu(\cdot|s^\prime))}Q_1(s^\prime,a^\prime) - \max_{a^\prime\in{\rm Support}(\mu(\cdot|s^\prime))}Q_2(s^\prime,a^\prime)  \right| \\
    &\le \gamma \max_{s,a} \|Q_1 - Q_2\|_\infty \\
    &= \gamma \|Q_1 - Q_2\|_\infty.
\end{aligned}
\end{equation*}
Combining the results together, we conclude that the MCB operator is a $\gamma$-contraction in the support region of the behavior policy, which naturally leads to the conclusion that any initial $Q$ function will converge to a unique fixed point by repeatedly applying $\mathcal{T}_{\rm MCB}$.
\end{proof}
In order to show Proposition \ref{prop:prop2}, we first have the following observations.
\begin{observation}
\label{obs:1}
In the support region of $\mu$, $Q_\mu$ is the fixed point of the Bellman operator $\mathcal{T}^\mu Q(s,a) := r(s,a) + \gamma \mathbb{E}_{s^\prime}\mathbb{E}_{a^\prime\sim\mu(\cdot|s^\prime)}[Q(s^\prime,a^\prime)]$.
\end{observation}
\begin{proof}
It is well known that the Bellman operator is a $\gamma$-contraction in the $\mathcal{L}_\infty$ norm, and hence for some arbitrarily initialized $Q$ function, it is guaranteed to converge to a unique fixed point by repeatedly applying $\mathcal{T}^\mu$. Hence, $Q_\mu$ is the $Q$ function of the behavior policy $\mu$ in ${\rm Support}(\mu)$.
\end{proof}
\begin{observation}
\label{obs:2}
In the support region of $\mu$, $Q_{\mu^*}$ is the fixed point of the Bellman optimal operator $\mathcal{T} Q(s,a) := r(s,a) + \gamma \mathbb{E}_{s^\prime}\left[\max_{a^\prime\sim\mu(\cdot|s^\prime)}Q(s^\prime,a^\prime)\right]$.
\end{observation}
\begin{proof}
The proof is similar to Observation \ref{obs:1}. Note that the resulting fixed point $Q_{\mu^*}$ is the $Q$ function of the optimal policy in the support region of the behavior policy. In offline RL setting, we only have a static logged dataset, then $Q_{\mu^*}$ is the $Q$ function of the optimal batch-constraint policy.
\end{proof}

\begin{proposition}[Behave at least as well as behavior policy]
\label{prop:prop2}
Denote $Q_{\mathrm{MCB}}$ as the unique fixed point acquired by the MCB operator, then in ${\rm Support}(\mu)$ we have: $Q_\mu \le Q_{\mathrm{MCB}}\le Q_{\mu^*}$, where $Q_\mu$ is the $Q$ function of the behavior policy and $Q_{\mu^*}$ is the $Q$ function of the optimal policy in the batch.
\end{proposition}

\begin{proof}
Denote $\mathcal{T}Q(s,a)$ as the Bellman optimal operator and $\mathcal{T}^\mu Q(s,a)$ as the Bellman operator. Since $Q_{\rm MCB}$ is the fixed point of the MCB operator, then if $a^\prime\in{\rm Support}(\mu)$, we have
\begin{equation}
\label{eq:1}
\begin{aligned}
    Q_{\rm MCB}=\mathcal{T}_{\rm MCB}Q(s,a) &= \mathcal{T}_2 Q(s,a) \\ &=r(s,a)+\gamma\mathbb{E}_{s^\prime}\left[ \max_{a^\prime\in{\rm Support}(\mu(\cdot|s^\prime))}Q(s^\prime,a^\prime) \right] \\
    &= \mathcal{T}Q(s,a) = Q_{\mu^*}(s,a).
\end{aligned}
\end{equation}
Furthermore, we have
\begin{equation}
\label{eq:2}
\begin{aligned}
    Q_{\rm MCB}=\mathcal{T}_{\rm MCB}Q(s,a) &= \mathcal{T}_2 Q(s,a) \\ &=r(s,a)+\gamma\mathbb{E}_{s^\prime}\left[ \max_{a^\prime\in{\rm Support}(\mu(\cdot|s^\prime))}Q(s^\prime,a^\prime) \right] \\
    &\ge r(s,a)+\gamma\mathbb{E}_{s^\prime} \mathbb{E}_{a^\prime\in{\rm Support}(\mu(\cdot|s^\prime))}[Q(s^\prime,a^\prime)] \\
    &= \mathcal{T}^\mu Q(s,a) = Q_{\mu}(s,a).
\end{aligned}
\end{equation}
If $a^\prime\notin {\rm Support}(\mu)$, we have
\begin{equation}
\label{eq:3}
\begin{aligned}
    Q_{\rm MCB}=\mathcal{T}_{\rm MCB}Q(s,a) &= \mathcal{T}_1 \mathcal{T}_2 Q(s,a) \\ &=r(s,a)+\gamma\mathbb{E}_{s^\prime}\left[ \max_{a^\prime\in{\rm Support}(\mu(\cdot|s^\prime))}Q(s^\prime,a^\prime) - \delta \right] \\
    &< \mathcal{T}Q(s,a) = Q_{\mu^*}(s,a).
\end{aligned}
\end{equation}
Moreover, by setting $0<\delta\le \mathbb{E}_{s^\prime}\left[ \max_{a^\prime \in {\rm Support}(\mu(\cdot|s^\prime))}Q(s^\prime,a^\prime) - \mathbb{E}_{a^\prime\sim{\rm Support}(\mu(\cdot|s^\prime))}[Q(s^\prime,a^\prime)] \right]$, we have
\begin{equation}
\label{eq:4}
\begin{aligned}
    Q_{\rm MCB}=\mathcal{T}_{\rm MCB}Q(s,a) &= \mathcal{T}_1 \mathcal{T}_2 Q(s,a) \\ &=r(s,a)+\gamma\mathbb{E}_{s^\prime}\left[ \max_{a^\prime\in{\rm Support}(\mu(\cdot|s^\prime))}Q(s^\prime,a^\prime) - \delta \right] \\
    &\ge r(s,a)+\gamma\mathbb{E}_{s^\prime} \mathbb{E}_{a^\prime\in{\rm Support}(\mu(\cdot|s^\prime))}[Q(s^\prime,a^\prime)] \\
    &= \mathcal{T}^\mu Q(s,a) = Q_{\mu}(s,a).
\end{aligned}
\end{equation}
Then by combining Eq. (\ref{eq:1}) and Eq. (\ref{eq:3}), the RHS holds. By combining Eq. (\ref{eq:2}) and Eq. (\ref{eq:4}), the LHS holds, which concludes the proof.
\end{proof}

\noindent\textbf{Remark:} Note that $\delta$ is a small positive number that can be arbitrarily small. As is explained in the main text (see Section 3.1), we subtract a small positive numer such that no OOD actions will be executed, considering the fact that the action is taken by following $\arg\max_{a\in\mathcal{A}}Q(s,a)$. We actively assign pseudo target values for the OOD actions, and if the allocated value is identical to the optimal $Q$ value, actions that lie outside of the span of the behavior policy $\mu$ may be chosen and executed, which may be unsafe. By subtracting a small $\delta$, we ensure that the OOD action training will not affect the learning for the optimal batch-constraint policy.

\begin{proposition}[Milder Pessimism]
\label{prop:prop3}
Suppose there exists an explicit policy constraint offline reinforcement learning algorithm such that the KL-divergence of the learned policy $\pi_p(\cdot|s)$ and the behavior policy $\mu(\cdot|s)$ is optimized to guarantee $\max\left(\mathrm{KL}(\mu, \pi_p), \mathrm{KL}(\pi_p, \mu) \right)\le\epsilon,\,\forall s$. Denote $\epsilon_\mu^{\pi_p} = \max_s|\mathbb{E}_{a\sim\pi_p}A^\mu(s,a)|$, where $A^\mu(s,a)$ is the advantage function. Then
\begin{equation}
    J(\pi_p) \ge J(\mu) - \dfrac{\sqrt{2} \gamma\epsilon_\mu^{\pi_p}}{(1-\gamma)^2}\sqrt{\epsilon},
\end{equation}
while for the policy $\pi_{\rm MCB}$ learned by applying the MCB operator, we have
\begin{equation}
    J(\pi_{\rm MCB}) \ge J(\mu).
\end{equation}
\end{proposition}

\begin{proof}
By utilizing the performance difference lemma (Lemma 6.1 in \cite{Kakade2002ApproximatelyOA}), the return difference between two policies $\pi^\prime$ and $\pi$ gives:
\begin{equation}
    \label{eq:performancedifference}
    J(\pi^\prime) - J(\pi) = \dfrac{1}{1-\gamma} \sum_{s} d_{\pi^\prime}(s) \sum_{a}\left[\pi^\prime(a|s)A^{\pi}(s,a)\right],
\end{equation}
where $d_{\pi^\prime}(s)$ is the probability distribution underlying the states present in the static dataset $\mathcal{D}$. Furthermore, based on the Corollary 1 in \cite{Achiam2017ConstrainedPO}, we have:
\begin{equation}
    \label{eq:performancediff}
    J(\pi^\prime) - J(\pi) \ge \dfrac{1}{1-\gamma} \sum_{s} d_{\pi}(s) \sum_{a}\left[\pi^\prime(a|s)A^{\pi}(s,a)\right] - \dfrac{2\gamma\epsilon_\pi^{\pi^\prime}}{(1-\gamma)^2}D_{\rm TV}^{d_\pi}(\pi^\prime, \pi),
\end{equation}
where $\epsilon_\pi^{\pi^\prime} = \max_s|\mathbb{E}_{a\sim\pi^\prime}A^\pi(s,a)|$ and $D_{\rm TV}^{d_\pi}(p,q)$ denotes the total variation between two distributions $p(\cdot)$ and $q(\cdot)$ over the distribution $d_\pi$, i.e., $D_{\rm TV}^{d_\pi}(p,q) = \frac{1}{2}\sum_{d_\pi}\sum_x |p(x) - q(x)|$. We use $D_{\rm TV}(p,q)$ to denote the total variation between two distributions $p(\cdot)$ and $q(\cdot)$, $D_{\rm TV}(p,q) = \frac{1}{2}\sum_x |p(x) - q(x)|$.

Following similar procedure as Proposition 2 in \cite{Wang2018ExponentiallyWI}, we have
\begin{equation}
\begin{aligned}
    J(\pi_p) - J(\mu) &\ge \dfrac{1}{1-\gamma} \sum_{s} d_{\mu}(s) \sum_{a}\pi_p(a|s)A^{\mu}(s,a) - \dfrac{2\gamma\epsilon_\mu^{\pi_p}}{(1-\gamma)^2}D_{\rm TV}^{d_\mu}(\pi_p, \mu) \\
    &\ge 0 - \dfrac{2\gamma\epsilon_\mu^{\pi_p}}{(1-\gamma)^2}D_{\rm TV}^{d_\mu}(\pi_p, \mu) \\
    &= - \dfrac{2\gamma\epsilon_\mu^{\pi_p}}{(1-\gamma)^2}D_{\rm TV}^{d_\mu}(\pi_p, \mu),
\end{aligned}
\end{equation}

Then, we have
\begin{equation}
    \label{eq:policyconstraint}
    \begin{aligned}
        J(\pi_p) - J(\mu) &\ge - \dfrac{2\gamma\epsilon_\mu^{\pi_p}}{(1-\gamma)^2}D_{\rm TV}^{d_\mu}(\pi_p, \mu) \\
        &= - \dfrac{2\gamma\epsilon_\mu^{\pi_p}}{(1-\gamma)^2}\sum_{s} d_{\mu}(s)\left[D_{\rm TV}(\pi_p(\cdot|s),\mu(\cdot|s))\right] \\
        & \ge -\dfrac{\sqrt{2}\gamma \epsilon_\mu^{\pi_p}}{(1-\gamma)^2} \sum_{s} d_{\mu}(s) \min\left(\sqrt{ {\rm KL}(\mu(\cdot|s), \pi_p(\cdot|s))}, \sqrt{ {\rm KL}(\pi_p(\cdot|s), \mu(\cdot|s))} \right) \\
        & \ge -\dfrac{\sqrt{2}\gamma \epsilon_\mu^{\pi_p}}{(1-\gamma)^2} \sum_{s} d_{\mu}(s) \max\left(\sqrt{ {\rm KL}(\mu(\cdot|s), \pi_p(\cdot|s))}, \sqrt{ {\rm KL}(\pi_p(\cdot|s), \mu(\cdot|s))} \right) \\
        &\ge -\dfrac{\sqrt{2}\gamma\epsilon_\mu^{\pi_p}}{(1-\gamma)^2} \sqrt{\epsilon}.
    \end{aligned}
\end{equation}
The third line above uses the Pinsker-Csiszatr inequality \cite{Csiszr2011InformationT}, and the fourth line holds due to the fact that $\max(a,b)\ge\min(a,b)$, and the final line holds due to the assumption we make, i.e., the KL-divergence between the learned policy and the behavior policy is optimized to guarantee $\max\left(\mathrm{KL}(\mu, \pi_p), \mathrm{KL}(\pi_p, \mu)\right)\le\epsilon,\,\forall s$. We hence conclude that the explicit policy constraint RL algorithms induce a safe policy improvement, i.e., $J(\pi_p) \ge J(\mu) - \mathcal{O}\left( \dfrac{1}{(1-\gamma)^2} \right)$.

For the policy induced by the MCB operator, however, we have
\begin{equation}
    \label{eq:performancemcb}
    \begin{aligned}
        J(\pi_{\rm MCB}) &= \mathbb{E}_{s\sim\mathcal{D}}\left[ V_{\pi_{\rm MCB}}(s) \right] \\
        & = \mathbb{E}_{s\sim\mathcal{D}}\mathbb{E}_{a\sim\pi_{\rm MCB}} \left[ Q_{\rm MCB}(s,a) \right] \\
        & = \mathbb{E}_{s\sim\mathcal{D}}\mathbb{E}_{a\sim{\rm Support}(\mu(\cdot|s))} \left[ Q_{\rm MCB}(s,a) \right] \\
        &\ge \mathbb{E}_{s\sim\mathcal{D}}\mathbb{E}_{a\sim{\rm Support}(\mu(\cdot|s))} \left[ Q_{\mu}(s,a) \right] \quad ({\rm By}\, {\rm using}\, {\rm Proposition}\, \ref{prop:prop2}) \\
        & = J(\mu).
    \end{aligned}
\end{equation}
The third line holds as the policy induced by the MCB operator will not execute actions that lie outside of the support region of behavior policy. That allows us to conclude the proof.
\end{proof}

\noindent\textbf{Remark:} Note that in Theorem 3.6 of \cite{Kumar2020ConservativeQF}, CQL also guarantees a safe policy improvement, i.e., $J(\pi_{\rm CQL})\ge J(\mu) - \mathcal{O}\left( \dfrac{1}{(1-\gamma)^2} \right)$, which is consistent with the lower bound of the explicit policy constraint methods. These methods are often overly pessimistic, while the policy learned by applying the MCB operator can improve the policy with a tighter lower bound.

\begin{proposition}
\label{pro:prop4}
Proposition \ref{prop:prop1} still holds for the practical MCB operator.
\end{proposition}

\begin{proof}
The practical MCB operator only modifies $\mathcal{T}_1$, i.e., the target value for the OOD actions. We only discuss $a^\prime \notin {\rm Support}(\mu(\cdot|s^\prime))$ here because for $a^\prime\in {\rm Support}(\mu(\cdot|s^\prime))$, the proof is identical as Proposition \ref{prop:prop1}. For $a^\prime \notin {\rm Support}(\mu(\cdot|s^\prime))$, let $Q_1$ and $Q_2$ be two arbitrary $Q$ functions, then
\begin{equation*}
    \begin{aligned}
        &\|\mathcal{\hat{T}}_{\rm MCB} Q_1 - \mathcal{\hat{T}}_{\rm MCB} Q_2 \|_\infty \\
        & = \|\hat{\mathcal{T}}_1\mathcal{T}_2 Q_1 - \hat{\mathcal{T}_1}\mathcal{T}_{2} Q_2 \|_\infty \\
        &= \max_{s,a} \left| \left( r(s,a) + \gamma \mathbb{E}_{s^\prime\sim\mathcal{D}}\mathbb{E}_{\{a_i^\prime\}^N\sim{\rm Support}(\hat{\mu})}\left[ \max_{a^\prime\in\{a_i^\prime\}^N}Q_1(s^\prime,a^\prime) \right] \right) \right . \\
        &\qquad \qquad \qquad \left . - \left( r(s,a) + \gamma \mathbb{E}_{s^\prime\sim\mathcal{D}}\mathbb{E}_{\{a_i^\prime\}^N\sim{\rm Support}(\hat{\mu})}\left[ \max_{a^\prime\in\{a_i^\prime\}^N}Q_2(s^\prime,a^\prime) \right] \right)  \right| \\
        &= \gamma \max_{s,a}\left| \mathbb{E}_{s^\prime\sim\mathcal{D}}\mathbb{E}_{\{a_i^\prime\}^N\sim{\rm Support}(\hat{\mu})}\left[ \max_{a^\prime\in\{a_i^\prime\}^N}Q_1(s^\prime,a^\prime) - \max_{a^\prime\in\{a_i^\prime\}^N}Q_2(s^\prime,a^\prime) \right]  \right| \\
        & \le \gamma \max_{s,a}\mathbb{E}_{s^\prime\sim\mathcal{D}}\mathbb{E}_{\{a_i^\prime\}^N\sim{\rm Support}(\hat{\mu})} \left| \max_{a^\prime\sim\{a_i^\prime\}^N}Q_1(s^\prime,a^\prime) - \max_{a^\prime\sim\{a_i^\prime\}^N}Q_2(s^\prime,a^\prime) \right| \\
        & \le \gamma\max_{s,a}\mathbb{E}_{s^\prime\sim\mathcal{D}}\mathbb{E}_{\{a_i^\prime\}^N} \| Q_1 - Q_2\|_\infty \\
        &= \gamma \|Q_1 - Q_2\|_\infty,
    \end{aligned}
\end{equation*}
where we use a fact that $\max_{a^\prime\in\{a_i^\prime\}^N}Q_1(s^\prime,a^\prime) - \max_{a^\prime\in\{a_i^\prime\}^N}Q_2(s^\prime,a^\prime) \le \|Q_1 - Q_2\|_\infty$. Its proof is quite straightforward. Let $\hat{a} = \arg\max_{\{a_i^\prime\}^N}Q_1(s^\prime,a_i^\prime)$, then
\begin{equation*}
\begin{aligned}
    \max_{a^\prime\in\{a_i^\prime\}^N}Q_1(s^\prime,a^\prime) - \max_{a^\prime\in\{a_i^\prime\}^N}Q_2(s^\prime,a^\prime) &= Q_1(s^\prime,\hat{a}) - \max_{a^\prime\in\{a_i^\prime\}^N}Q_2(s^\prime,a^\prime) \\
    &\le Q_1(s^\prime,\hat{a}) - Q_2(s^\prime,\hat{a}) \le \|Q_1 - Q_2\|_\infty.
\end{aligned}
\end{equation*}
Therefore, we conclude that Proposition \ref{prop:prop1} still holds for the practical MCB operator.
\end{proof}

\begin{proposition}[No erroneous overestimation will occur]
\label{pro:prop5}
Assuming that $\sup_s D_{\rm TV}( \hat \mu( \cdot|s) \parallel \mu( \cdot|s)) \leq \epsilon < \frac{1}{2}$, we have
\begin{align*}
    \mathbb{E}_{\{a_i^\prime\}^N\sim \hat{\mu}(\cdot|s)} \left[\max_{a^\prime\in\{a_i^\prime\}^N}Q(s,a^\prime)\right] \leq \max_{a^\prime \in  \operatorname{Support}(\mu(\cdot|s))} Q(s,a^\prime) + (1 - (1 - 2\epsilon)^N) \dfrac{r_{\rm max}}{1-\gamma}.
\end{align*}
\end{proposition}

\begin{proof}
The total variation divergence $D_{TV}(p\|q) = \frac{1}{2}\sum_{x}|p(x) - q(x)|$ for some distributions $p(\cdot)$ and $q(\cdot)$, then we have
\begin{equation}
\begin{aligned}
    1 > 2\epsilon  \geq 2\sup_s D_{\rm TV}( \hat \mu( \cdot | s) \parallel \mu( \cdot | s)) &\geq \sum_a |\hat \mu(a | s) - \mu(a | s)| \\
    &= \sum_{a\in{\rm Support}(\mu(\cdot|s))} |\hat \mu(a | s) - \mu(a | s)| + \sum_{a\notin{\rm Support}(\mu(\cdot|s))} |\hat \mu(a | s) - \mu(a | s)| \\
    &\geq  \sum_{a \notin \operatorname{Support}(\mu(\cdot| s))} \hat \mu(a | s),
\end{aligned}
\end{equation}
where we use the fact that $\mu(a|s)=0$ if $a\notin{\rm Support}(\mu(\cdot|s))$. Denote $Q_{\rm max}$ as the maximum value of the OOD actions, and $Q_{\rm max}\le \frac{r_{\rm max}}{1-\gamma}$. Thus, we have
\begin{align*}
    & \mathbb{E}_{\{a_i^\prime\}^N\sim \hat{\mu}(\cdot|s)} \left[\max_{a^\prime\in\{a_i^\prime\}^N}Q(s,a^\prime)\right] \\
    &\leq \mathbb{P}\left(\bigcap_{i} \left\{ a_i^\prime \in  \operatorname{Support}(\mu(\cdot| s))\right\}\right) 
    \max_{a^\prime \in  \operatorname{Support}(\mu(\cdot| s))} Q(s,a^\prime) + \mathbb{P}\left(\bigcup_i \left\{a_i^\prime \notin  \operatorname{Support}(\mu(\cdot| s)) \right\}\right) Q_{\rm max} \\
    &\leq  \max_{a^\prime \in  \operatorname{Support}(\mu(\cdot| s))} Q(s,a^\prime) + \left(1 - \mathbb{P}( a_1^\prime \in  \operatorname{Support}(\mu(\cdot| s))
    )^N\right) \dfrac{r_{\rm max}}{1-\gamma} \\
    &\leq  \max_{a^\prime \in  \operatorname{Support}(\mu(\cdot| s))} Q(s,a^\prime) + \left(1 - (1 - 2\epsilon)^N\right) \dfrac{r_{\rm max}}{1-\gamma}.
\end{align*}
\end{proof}

\noindent\textbf{Remark:} Note that we empirically fit a behavior policy $\hat{\mu}(\cdot|s)$ as we usually do not have prior knowledge about the true behavior policy $\mu$. The conclusion says that for the practical MCB operator and empirical behavior policy $\hat{\mu}(\cdot|s)$, even if OOD actions are sampled from the $\hat{\mu}(\cdot|s)$, the extrapolation error is under the scale of $\left(1-(1 - 2\epsilon)^N\right) \dfrac{r_{\rm max}}{1-\gamma}$, indicating the fact that no severe overestimation will occur. If the empirical behavior policy well fits the true behavior policy, then the error term will be arbitrarily small and the value estimate will well approximate the maximum $Q$ value in the batch. We assume that the total variation divergence between $\hat{\mu}(\cdot|s)$ and $\mu$ is bounded by $\frac{1}{2}$, which is easy to satisfy in practice, e.g., fit the empirical behavior policy via supervised learning.

\section{Deterministic MCQ}
\label{sec:determcq}
As mentioned in the main text, the practical MCB operator is pluggable and can be combined with wide range of off-policy online RL methods. In this section, we present a deterministic version of MCQ by incorporating the practical MCB operator with TD3 \cite{Fujimoto2018AddressingFA}. 

Similarly, we train a CVAE $G_\omega$ which is made up of an encoder $E_\xi(s,a)$ and a decoder $D_\psi(s,z)$ as the generative model, which is optimized by:
\begin{equation}
    \label{eq:appendixCVAE}
    \mathcal{L}_{\rm CVAE} = \mathbb{E}_{(s,a)\sim\mathcal{D},z\sim E_\xi(s,a)}\left[ (a - D_\psi(s,z))^2 + {\rm KL}\left(E_\xi(s,a), \mathcal{N}(0,{\bf I})\right) \right].
\end{equation}

The objective function for the critics gives:
\begin{equation}
\label{eq:deterministiccriticloss}
    \mathcal{L}_{\rm critic} = \lambda \mathbb{E}_{(s,a,r,s^\prime)\sim \mathcal{D}}\left[ (Q_{\theta_i}(s,a) - y)^2 \right] + (1-\lambda){\color{red} \mathbb{E}_{s^{\rm in}\sim\mathcal{D}}\left[ (Q_{\theta_i}(s^{\rm in}, \pi(s^{\rm in})) - y^\prime)^2 \right]},
\end{equation}
where $\theta_i,i=1,2$, is the parameter of the critic network, and the target value for the in-distribution action gives:
\begin{equation}
    \label{eq:deterministictruetarget}
    y = r(s,a) + \gamma\left[\min_{i=1,2}Q_{\theta_i^\prime}(s^\prime,\pi_{\phi^\prime}(s^\prime))\right],
\end{equation}
where $\theta_i^\prime, i=1,2$ is the parameter of the lagging target network, and $\phi^\prime$ is the parameter of the actor network. For OOD actions, the target value gives:
\begin{equation}
    \label{eq:deterministicfaketarget}
    y^\prime = \min_{j=1,2}\mathbb{E}_{\{a_i'\}^N\sim\hat{\mu}}\left[ \max_{a^\prime\sim\{a_i'\}^N}Q_{\theta_j}(s^{\rm in},a^\prime) \right].
\end{equation}
The actor is then optimized by $\max_{\phi} \mathbb{E}_{s\sim\mathcal{D}}\left[Q_{\theta_1}(s,\pi_{\phi}(s))\right]$. The full procedure of the deterministic MCQ is presented in Algorithm \ref{alg:algmcqdeerministic}.

\begin{algorithm}
\caption{Deterministic version of Mildly Conservative $Q$-learning (MCQ)}
\label{alg:algmcqdeerministic}
\begin{algorithmic}[1] 
\STATE Initialize CVAE $G_\omega$, critic networks $Q_{\theta_1}, Q_{\theta_2}$ and actor network $\pi_{\phi}$ with random parameters 
\STATE Initialize target networks $\theta_1^\prime \leftarrow \theta_1, \theta_2^\prime \leftarrow \theta_2, \phi^\prime\leftarrow \phi$ and offline replay buffer $\mathcal{D}$.
\FOR{$t$ = 1 to $T$}
\STATE Sample a mini-batch $B = \{(s,a,r,s^\prime,d)\}$ from $\mathcal{D}$, where $d$ is the done flag
\STATE Train CVAE via minimizing Eq. (\ref{eq:appendixCVAE})
\STATE Get target value: $y = r(s,a) + \gamma\left[\min_{i=1,2}Q_{\theta_i^\prime}(s^\prime,\pi_{\phi^\prime}(s^\prime)\right]$
\STATE Set $s^{\rm in} = \{s,s^\prime\}$ and compute the target value for the OOD actions via Eq. (\ref{eq:deterministicfaketarget})
\STATE Update critic $\theta_i$ with gradient descent via minimizing Eq. (\ref{eq:deterministiccriticloss})
\STATE Update $\phi$ by policy gradient: $\nabla_\phi J(\phi) = |B|^{-1}\sum \nabla_a Q_{\theta_1}(s,a)|_{a = \pi_\phi(s)} \nabla_\phi \pi_\phi(s)$
\STATE Update target networks: $\phi^\prime \leftarrow \tau \phi + (1-\tau)\phi^\prime, \theta_i^\prime\leftarrow\tau\theta_i + (1-\tau)\theta_i^\prime,\, i=1,2$
\ENDFOR
\end{algorithmic}
\end{algorithm}

Next, we empirically examine the performance of the deterministic MCQ by conducting experiments on D4RL \cite{Fu2020D4RLDF} MuJoCo locomotion tasks. As we want to focus on the benefits of MCQ on non-expert datasets, and considering the fact that all baselines can achieve good performance on expert datasets, we only report the performance comparison on non-expert datasets (random, medium, medium-replay, and medium-expert). The results are summarized in Table \ref{tab:deterministicmcq}. For the deterministic MCQ, we adopt identical $\lambda$ and $N$ as vanilla MCQ (see Table \ref{tab:parameters}) except that we use $\lambda=0.4$ on hopper-medium-expert-v2. We observe that the deterministic MCQ has a good performance on many non-expert datasets, which is consistent with vanilla MCQ (which is built upon SAC). The deterministic MCQ has an average score of 68.4 on non-expert datasets and the vanilla MCQ has an average score of 72.8. Note that it is not suitable to compare BC against the deterministic MCQ (though deterministic MCQ still outperforms BC on all of the datasets), since the actor in our BC is not deterministic.

All these validate our claim that the practical MCB operator is pluggable and can be combined with any off-policy online RL algorithm to make it work in the offline setting. Note that it is also very interesting to report all the numerical results based on \cite{Agarwal2021DeepRL}, which we omit here.

\begin{table}[htb]
  \caption{Normalized average score comparison of MCQ against baseline methods on D4RL benchmarks over the final 10 evaluations. 0 corresponds to a random policy and 100 corresponds to an expert policy. The experiments are run on MuJoCo "-v2" datasets over 4 random seeds. r = random, m = medium, m-r = medium-replay, m-e = medium-expert, e = expert. MCQ (TD3) denotes the deterministic version of MCQ that built upon TD3.}
  \label{tab:deterministicmcq}
  \small
  \centering
  \setlength{\tabcolsep}{2.8pt}
  \begin{tabular}{@{}lllllllll@{}}
    \toprule
    Task Name  & BC & SAC & CQL & UWAC & TD3+BC & IQL & MCQ (TD3) \\
    \midrule
    halfcheetah-r & 2.2$\pm$0.0 & \colorbox{yellow}{29.7}$\pm$1.4 & 17.5$\pm$1.5 & 2.3$\pm$0.0 & 11.0$\pm$1.1 & 13.1$\pm$1.3 & \colorbox{yellow}{23.6}$\pm$0.8 \\
    hopper-r & 3.7$\pm$0.6 & 9.9$\pm$1.5 & 7.9$\pm$0.4 & 2.7$\pm$0.3 & 8.5$\pm$0.6 & 7.9$\pm$0.2 & \colorbox{yellow}{31.0}$\pm$1.7 \\
    walker2d-r & 1.3$\pm$0.1 & 0.9$\pm$0.8 & 5.1$\pm$1.3 & 2.0$\pm$0.4 & 1.6$\pm$1.7 & 5.4$\pm$1.2 & \colorbox{yellow}{10.3}$\pm$6.8 \\
    halfcheetah-m & 43.2$\pm$0.6 & \colorbox{yellow}{55.2}$\pm$27.8 & 47.0$\pm$0.5 & 42.2$\pm$0.4 & 48.3$\pm$0.3 & 47.4$\pm$0.2 & \colorbox{yellow}{58.3}$\pm$1.3 \\
    hopper-m & 54.1$\pm$3.8 & 0.8$\pm$0.0 & 53.0$\pm$28.5 & 50.9$\pm$4.4 & 59.3$\pm$4.2 & \colorbox{yellow}{66.2}$\pm$5.7 & \colorbox{yellow}{73.6}$\pm$10.3  \\
    walker2d-m & 70.9$\pm$11.0 & -0.3$\pm$0.2 & 73.3$\pm$17.7 & 75.4$\pm$3.0 & \colorbox{yellow}{83.7}$\pm$2.1 & 78.3$\pm$8.7 & \colorbox{yellow}{88.4}$\pm$1.3 \\
    halfcheetah-m-r & 37.6$\pm$2.1 & 0.8$\pm$1.0 & \colorbox{yellow}{45.5}$\pm$0.7 & 35.9$\pm$3.7 & 44.6$\pm$0.5 & 44.2$\pm$1.2 & \colorbox{yellow}{51.5}$\pm$0.2 \\
    hopper-m-r & 16.6$\pm$4.8 & 7.4$\pm$0.5 & 88.7$\pm$12.9 & 25.3$\pm$1.7 & 60.9$\pm$18.8 & \colorbox{yellow}{94.7}$\pm$8.6 & \colorbox{yellow}{99.5}$\pm$1.7 \\
    walker2d-m-r & 20.3$\pm$9.8 & -0.4$\pm$0.3 & \colorbox{yellow}{81.8}$\pm$2.7 & 23.6$\pm$6.9 & \colorbox{yellow}{81.8}$\pm$5.5 & 73.8$\pm$7.1 & \colorbox{yellow}{83.3}$\pm$1.9 \\
    halfcheetah-m-e & 44.0$\pm$1.6 & 28.4$\pm$19.4 & 75.6$\pm$25.7 & 42.7$\pm$0.3 & \colorbox{yellow}{90.7}$\pm$4.3 & \colorbox{yellow}{86.7}$\pm$5.3 & \colorbox{yellow}{85.4}$\pm$3.4 \\
    hopper-m-e & 53.9$\pm$4.7 & 0.7$\pm$0.0 & \colorbox{yellow}{105.6}$\pm$12.9 & 44.9$\pm$8.1 & 98.0$\pm$9.4 & 91.5$\pm$14.3 & \colorbox{yellow}{106.1}$\pm$2.3 \\
    walker2d-m-e & 90.1$\pm$13.2 & 1.9$\pm$3.9 & \colorbox{yellow}{107.9}$\pm$1.6 & 96.5$\pm$9.1 & \colorbox{yellow}{110.1}$\pm$0.5 & \colorbox{yellow}{109.6}$\pm$1.0 & \colorbox{yellow}{110.3}$\pm$0.1 \\
    \midrule
    Average & 36.5 & 11.3 & 59.1 & 37.0 & 58.2 & 59.9 & \colorbox{yellow}{68.4} \\
    \bottomrule
  \end{tabular}
\end{table}

\section{Experimental Details and Parameter Setup}
\label{sec:detail}
In this section, we first give a brief introduction to the D4RL benchmarks \cite{Fu2020D4RLDF}. Then we present our implementation details and experimental details.
\subsection{D4RL Benchmarks}
In the paper, we evaluate MCQ on the D4RL MuJoCo-Gym domain, which contains three environments (halfcheetah, hopper, walker2d), and five types of datasets (random, medium, medium-replay, medium-expert, expert). \textbf{Random} datasets are gathered by a random policy. \textbf{Medium} datasets contain experience from an early-stopped SAC policy. \textbf{Medium-replay} datasets are collected during the training process of the "medium" SAC policy. \textbf{Medium-expert} datasets are formed by combining the suboptimal samples and the expert samples. \textbf{Expert} datasets are made up of expert trajectories. The illustration of MuJoCo tasks is shown in Figure \ref{fig:mujocoviaual}.

\begin{figure}[!htb]
    \centering
    \subfigure[HalfCheetah]{
    \label{fig:halfcheetah}
    \includegraphics[scale=0.295]{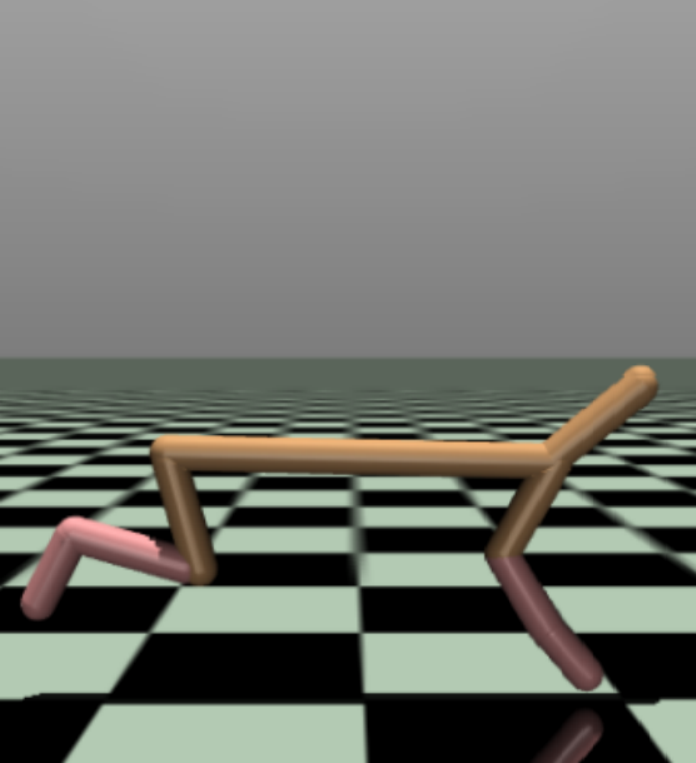}
    }
    \subfigure[Hopper]{
    \label{fig:hopper}
    \includegraphics[scale=0.3]{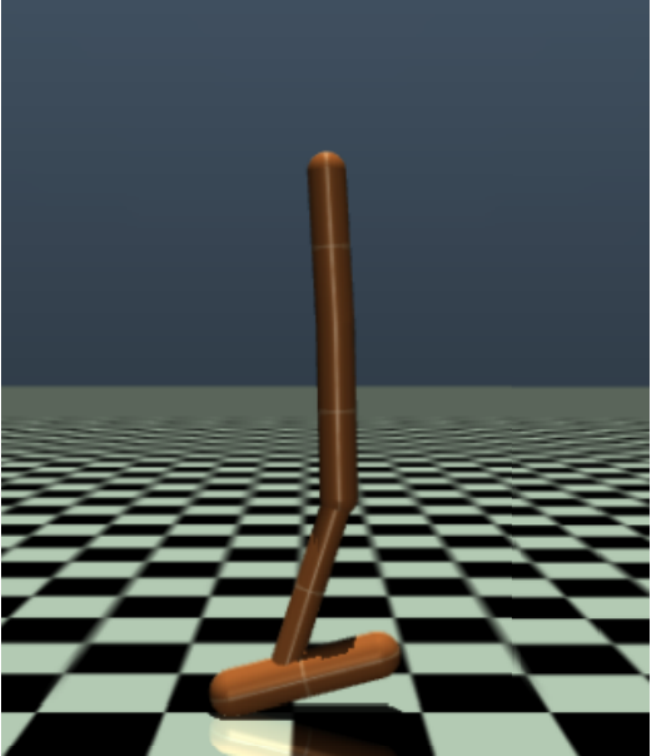}
    }
    \subfigure[Walker2d]{
    \label{fig:walker2d}
    \includegraphics[scale=0.3]{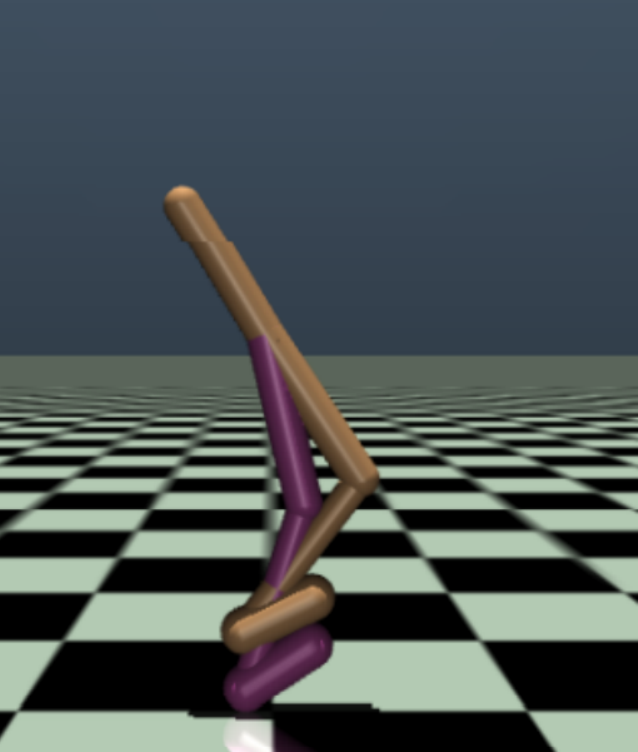}
    }
    \caption{MuJoCo datasets. We conduct experiments on halfcheetah, hopper, and walker2d tasks.}
    \label{fig:mujocoviaual}
\end{figure}

D4RL offers a metric, normalized score, to evaluate the performance of the offline RL algorithm, which is calculated by:
\begin{equation*}
    \mathrm{score} = \dfrac{\mathrm{average}\,\mathrm{return}-\mathrm{return}\,\mathrm{of}\,\mathrm{the}\,\mathrm{random}\,\mathrm{policy}}{\mathrm{return}\,\mathrm{of}\,\mathrm{the}\,\mathrm{expert}\,\mathrm{policy} - \mathrm{return}\,\mathrm{of}\,\mathrm{the}\,\mathrm{random}\,\mathrm{policy}}\times 100.
\end{equation*}

Note that if the normalized score equals to 0, that indicates that the learned policy has a similar performance as the random policy, while 100 corresponds to an expert policy. In D4RL, different types of datasets (e.g., random, medium, expert) share the identical reference minimum score and reference maximum score. We summarize the reference score for each environment in Table \ref{tab:referencescore}.

\begin{table}[!htb]
  \caption{The referenced min score and max score for the MuJoCo dataset in D4RL.}
  \label{tab:referencescore}
  \vspace{0.2cm}
  \centering
  \begin{tabular}{llcc}
    \toprule
    Domain & Task Name   & Reference Min Score & Reference Max Score \\
    \midrule
    MuJoCo & halfcheetah & $-$280.18 & 12135.0 \\
    MuJoCo & hopper & $-$20.27 & 3234.3 \\
    MuJoCo & Walker2d & 1.63 & 4592.3 \\
    \bottomrule
  \end{tabular}
\end{table}

\subsection{Implementation Details}
\label{sec:implementation}

\noindent\textbf{Baselines.} We conduct experiments on the recently released MuJoCo "-v2" datasets. We adopt behavior cloning (BC), Soft Actor-Critic (SAC) \cite{Haarnoja2018SoftAO}, and recently proposed offline RL methods, CQL \cite{Kumar2020ConservativeQF}, UWAC \cite{Wu2021UncertaintyWA}, TD3+BC \cite{Fujimoto2021AMA}, and IQL \cite{kostrikov2022offline} as baseline methods. The results of CQL are obtained by running the official codebase (\href{https://github.com/aviralkumar2907/CQL}{https://github.com/aviralkumar2907/CQL}). For UWAC, we rerun it on MuJoCo "-v2" datasets using the official codebase (\href{https://github.com/apple/ml-uwac}{https://github.com/apple/ml-uwac}) since the results reported in its original paper are not obtained on "-v2" datasets. For TD3+BC, we actually can reproduce its results, and thus we take the results reported in its original paper (Table 7 in the appendix) directly. For IQL, we take its results on medium, medium-replay, and medium-expert from its original paper directly. We run IQL on random and expert datasets with its official codebase (\href{https://github.com/ikostrikov/implicit_q_learning}{https://github.com/ikostrikov/implicit\_q\_learning}). All methods are run over 4 different random seeds and the normalized average scores are reported in the main text.

\noindent\textbf{MCQ implementation details.} We build the practical MCB operator upon Soft Actor-Critic (SAC), and modify its critic objective function by adding an auxiliary loss such that OOD actions can be trained. For some datasets, we subtract the pseudo target value a small positive number (as in Definition \ref{def:mcb}) since we find it behaves better. We also stop overly estimated value from backpropagating for some of the datasets. We additionally train a generative model, CVAE. For each training epoch, we sample a mini-batch from the logged dataset, and train the CVAE. Next, we leverage the critic networks for estimating the $Q$-values and their corresponding target values upon in-distribution state-action pairs. Furthermore, we sample $N$ actions from the CVAE based on the current state and the next state, and sample $N$ actions from the learned policy. We then construct the pseudo target values for OOD actions. Finally, we optimize the learned policy with the standard way SAC does. We detail the default parameter setup for SAC and CVAE in Table \ref{tab:hyperparameter}.

\begin{table*}
\caption{Hyperparameters setup for MCQ.}
\label{tab:hyperparameter}
\centering
\begin{tabular}{lc}
\toprule
\textbf{Hyperparameter} & \textbf{Value} \\
\midrule
SAC  & \\
\qquad Actor network  & $(400,400)$ \\
\qquad Critic network & $(400,400)$ \\
\qquad Batch size & $256$ \\
\qquad Critic learning rate & $3\times 10^{-4}$ \\
\qquad Actor learning rate & $3\times 10^{-4}$ \\
\qquad Optimizer & Adam \cite{Kingma2015AdamAM} \\
\qquad Discount factor & $0.99$ \\
\qquad Reward scale & $1$ \\
\qquad Maximum log std & $2$ \\
\qquad Minimum log std & $-20$ \\
\qquad Use automatic entropy tuning & Yes \\
\qquad Target update rate & $5\times 10^{-3}$ \\
\midrule
CVAE & \\
\qquad Encoder hidden dimension & $750$ \\
\qquad Decoder hidden dimension & $750$ \\
\qquad Hidden layers & $2$ \\
\qquad CVAE learning rate & $1\times 10^{-3}$ \\
\qquad Batch size & $256$ \\
\qquad Latent dimension & $2\times$ action dimension \\
\bottomrule
\end{tabular}
\end{table*}

We set the number of sampled actions $N=10$ by default as we experimentally find that MCQ is insensitive to $N$ in a wide range. Across all of our experiments, we tune the weighting coefficient $\lambda$ across $\{0.3, 0.5, 0.6, 0.7, 0.8, 0.9, 0.95\}$ with grid search. We summarize the hyperparameters we use for running the MuJoCo datasets in Table \ref{tab:parameters}, where we find that generally $0.7\le\lambda<1$ is suitable for most of the datasets. For some datasets with narrow distributions, e.g., expert datasets and medium-expert datasets, a comparatively small $\lambda$ is better because actions sampled from the learned policy are more likely to lie in the OOD region in these tasks, and larger weights are expected to be assigned to OOD actions training loss. Again, as we have explained in the main text, one ought not to use too small $\lambda$ (and absolutely cannot set $\lambda=0$) as our first priority is always training the in-distribution data in a good manner.

\noindent\textbf{Offline-to-Online details.} For offline-to-online experiments, we directly put the online samples into the offline replay buffer after offline initialization. Both offline transitions and online transitions are sampled equally during online interaction and training. The hyperparameters of all methods are kept unchanged on both offline stage and online stage. The results of baseline methods are acquired by running their official codebases, i.e., CQL (\href{https://github.com/aviralkumar2907/CQL}{https://github.com/aviralkumar2907/CQL}), TD3+BC (\href{https://github.com/sfujim/TD3\_BC}{https://github.com/sfujim/TD3\_BC}), IQL (\href{https://github.com/ikostrikov/implicit\_q\_learning}{https://github.com/ikostrikov/implicit\_q\_learning}), AWAC (\href{https://github.com/vitchyr/rlkit}{https://github.com/vitchyr/rlkit}). For those involve normalization over states, we normalize the online samples with the mean and standard deviation calculated based on the offline samples. All methods are run over 4 different random seeds. We choose a subset of tasks for offline-to-online fine-tuning different from IQL and AWAC to ensure that our empirical experiments on offline-to-online fine-tuning are consistent to the offline experiments on MuJoCo datasets.

\begin{table}[htb]
  \caption{Detailed hyperparameters used in MCQ, where we conduct experiments on D4RL MuJoCo-Gym "-v2" datasets.}
  \label{tab:parameters}
  \small
  \centering
  \begin{tabular}{l|c|c}
    \toprule
    Task Name  & weighting coefficient $\lambda$ & number of sampled actions $N$ \\
    \midrule
    halfcheetah-random-v2 & 0.95 & 10 \\
    hopper-random-v2 & 0.6 & 10 \\
    walker2d-random-v2 & 0.8 & 10 \\
    halfcheetah-medium-v2 & 0.95 & 10 \\
    hopper-medium-v2 & 0.7 & 10 \\
    walker2d-medium-v2 & 0.9 & 10 \\
    halfcheetah-medium-replay-v2 & 0.95 & 10 \\
    hopper-medium-replay-v2 & 0.9 & 10 \\
    walker2d-medium-replay-v2 & 0.9 & 10 \\
    halfcheetah-medium-expert-v2 & 0.7 & 10 \\
    hopper-medium-expert-v2 & 0.5 & 10 \\
    walker2d-medium-expert-v2 & 0.8 & 10 \\
    halfcheetah-expert-v2 & 0.5 & 10 \\
    hopper-expert-v2 & 0.5 & 10 \\
    walker2d-expert-v2 & 0.3 & 10 \\
    \bottomrule
  \end{tabular}
\end{table}

\section{Comparison of MCQ against More Baselines}
In this section, we extensively compare MCQ against severe recent strong offline RL methods on D4RL MuJoCo datasets, including BCQ \cite{Fujimoto2019OffPolicyDR}, BEAR \cite{Kumar2019StabilizingOQ}, MOPO \cite{Yu2020MOPOMO}, Decision Transformer (DT) \cite{Chen2021DecisionTR}, CDC \cite{Fakoor2021ContinuousDC}, and PBRL \cite{bai2022pessimistic}. BCQ aims at learning a batch-constraint policy and it also involves a generative model, i.e., CVAE. BEAR lies in the category of explicit policy constraint offline RL methods. MOPO is a well-known model-based offline RL algorithm that leverages the learned dynamics models for uncertainty quantification. Decision Transformer (DT) is also a recent model-based offline RL method, and it relies on the transformer for sequential modeling. CDC adds critic regularization for the critics and the explicit policy constraint for the actor. PBRL is an ensemble-based offline RL method, which also quantifies the estimation uncertainty. We conduct experiments on D4RL \cite{Fu2020D4RLDF} MuJoCo locomotion datasets. Note that the results of BCQ is obtained by running its official codebase (\href{https://github.com/sfujim/BCQ}{https://github.com/sfujim/BCQ}), and the results of BEAR, MOPO, and PBRL are taken directly from \cite{bai2022pessimistic}. The results of Decision Transformer are taken from \cite{Fujimoto2021AMA} (Table 7 in the appendix). Since CDC does not report standard deviation in its original paper, we omit them and directly report its mean performance. We summarize the normalized average score of each method in Table \ref{tab:morebaselines}.

As shown, MCQ outperforms these baseline methods on most of the non-expert datasets, often by a large margin, and is competitive with them on expert datasets. Note that RBRL also leverages OOD sampling in its formulation, The main difference between MCQ and PBRL lies in the fact that PBRL adopts OOD sampling for penalizing the estimation uncertainty while we construct pseudo target values for OOD actions. Moreover, PBRL utilizes critic ensemble for uncertainty quantification while MCQ only adopts double critics. MCQ is much more computationally efficient than PBRL. One can also find that MCQ surpasses PBRL on many datasets, especially on non-expert datasets.

\begin{table}[htb]
  \caption{Normalized average score comparison of MCQ against baseline methods on D4RL benchmarks. 0 corresponds to a random policy and 100 corresponds to an expert policy. The experiments are run on MuJoCo "-v2" datasets over 4 random seeds. r = random, m = medium, m-r = medium-replay, m-e = medium-expert, e = expert.}
  \label{tab:morebaselines}
  \small
  \centering
  \setlength{\tabcolsep}{3.6pt}
  \begin{tabular}{@{}llllllll@{}}
    \toprule
    Task Name  & BCQ & BEAR & MOPO & DT & CDC & PBRL & MCQ (ours) \\
    \midrule
    halfcheetah-r & 2.2$\pm$0.0 & 2.3$\pm$0.0 & \colorbox{yellow}{35.9}$\pm$2.9 & - & 27.4 & 11.0$\pm$5.8 & \colorbox{yellow}{28.5}$\pm$0.6 \\
    hopper-r & 7.8$\pm$0.6 & 3.9$\pm$2.3 & 16.7$\pm$12.2 & - & 14.8 & 26.8$\pm$9.3 & \colorbox{yellow}{31.8}$\pm$0.5 \\
    walker2d-r & 4.9$\pm$0.1 & 12.8$\pm$10.2 & 4.2$\pm$5.7 & - & 7.2 & 8.1$\pm$4.4 & \colorbox{yellow}{17.0}$\pm$3.0 \\
    halfcheetah-m & 46.6$\pm$0.4 & 43.0$\pm$0.2 & \colorbox{yellow}{73.1}$\pm$2.4 & 42.6$\pm$0.1 & 46.1 & 57.9$\pm$1.5 & \colorbox{yellow}{64.3}$\pm$0.2 \\
    hopper-m & 59.4$\pm$8.3 & 51.8$\pm$4.0 & 38.3$\pm$34.9 & 67.6$\pm$1.0 & 60.4 & \colorbox{yellow}{75.3}$\pm$31.2 & \colorbox{yellow}{78.4}$\pm$4.3 \\
    walker2d-m & 71.8$\pm$7.2 & -0.2$\pm$0.1 & 41.2$\pm$30.8 & 74.0$\pm$1.4 & 82.1 & \colorbox{yellow}{89.6}$\pm$0.7 & \colorbox{yellow}{91.0}$\pm$0.4 \\
    halfcheetah-m-r & 42.2$\pm$0.9 & 36.3$\pm$3.1 & \colorbox{yellow}{69.2}$\pm$1.1 & 36.6$\pm$0.8 & 44.7 & 45.1$\pm$8.0 & \colorbox{yellow}{56.8}$\pm$0.6 \\
    hopper-m-r & 60.9$\pm$14.7 & 52.2$\pm$19.3 & 32.7$\pm$9.4 & 82.7$\pm$7.0 & 55.4 & \colorbox{yellow}{100.6}$\pm$1.0 & \colorbox{yellow}{101.6}$\pm$0.8 \\
    walker2d-m-r & 57.0$\pm$9.6 & 7.0$\pm$7.8 & 73.7$\pm$9.4 & 66.6$\pm$3.0 & 23.0 & 77.7$\pm$14.5 & \colorbox{yellow}{91.3}$\pm$5.7 \\
    halfcheetah-m-e & \colorbox{yellow}{95.4}$\pm$2.0 & 46.0$\pm$4.7 & 70.3$\pm$21.9 & 86.8$\pm$1.3 & 59.6 & \colorbox{yellow}{92.3}$\pm$1.1 & 87.5$\pm$1.3 \\
    hopper-m-e & 106.9$\pm$5.0 & 50.6$\pm$25.3 & 60.6$\pm$32.5 & 107.6$\pm$1.8 & 86.9 & \colorbox{yellow}{110.8}$\pm$0.8 & \colorbox{yellow}{111.2}$\pm$0.1 \\
    walker2d-m-e & \colorbox{yellow}{107.7}$\pm$3.8 & 22.1$\pm$44.9 & 77.4$\pm$27.9 & \colorbox{yellow}{108.1}$\pm$0.2 & 70.9 & \colorbox{yellow}{110.1}$\pm$0.3 & \colorbox{yellow}{114.2}$\pm$0.7 \\
    \midrule
    Average Above & 55.2 & 27.3 & 49.4 & - & 48.2 & \colorbox{yellow}{67.1} & \colorbox{yellow}{72.8} \\
    \midrule
    halfcheetah-e & 89.9$\pm$9.6 & \colorbox{yellow}{92.7}$\pm$0.6 & 81.3$\pm$21.8 & - & 82.1 & \colorbox{yellow}{92.4}$\pm$1.7 & \colorbox{yellow}{96.2}$\pm$0.4 \\
    hopper-e & \colorbox{yellow}{109.0}$\pm$4.0 & 54.6$\pm$21.0 & 62.5$\pm$29.0 & - & \colorbox{yellow}{102.8} & 110.5$\pm$0.4 & \colorbox{yellow}{111.4}$\pm$0.4 \\
    walker2d-e & \colorbox{yellow}{106.3}$\pm$5.0 & \colorbox{yellow}{106.6}$\pm$6.8 & 62.4$\pm$3.2 & - & 87.5 & \colorbox{yellow}{108.3}$\pm$0.3 & \colorbox{yellow}{107.2}$\pm$1.1 \\
    \midrule
    Total Average & 64.5 & 38.8 & 53.3 & - & 56.7 & \colorbox{yellow}{74.4} & \colorbox{yellow}{79.2} \\
    \bottomrule
  \end{tabular}
\end{table}

\section{Wider Evidence on Value Estimation of MCQ}
We first want to include the $Q$-value estimation with respect to the $\lambda$ and $N$ on halfcheetah-medium-v2, which we omit in the main text due to the space limit. The results are shown in Figure \ref{fig:valueestimationhalfcheetah}, where we can observe that the results resemble those on hopper-medium-replay-v2, e.g., the $Q$ value estimate collapses with smaller $\lambda$ and is comparatively insensitive to $N$ in a wide range.

\begin{figure}
    \centering
    \subfigure[$Q$ value w.r.t. $\lambda$]{
    \label{fig:halfcheetahlambdaq}
    \includegraphics[scale=0.6]{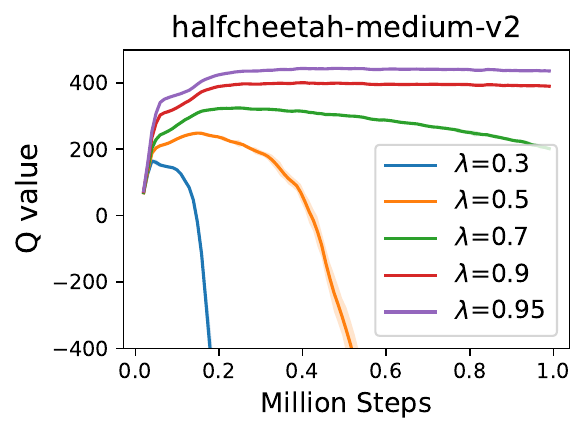}
    }
    \subfigure[$Q$ value w.r.t. $N$]{
    \label{fig:halfcheetahlambdan}
    \includegraphics[scale=0.6]{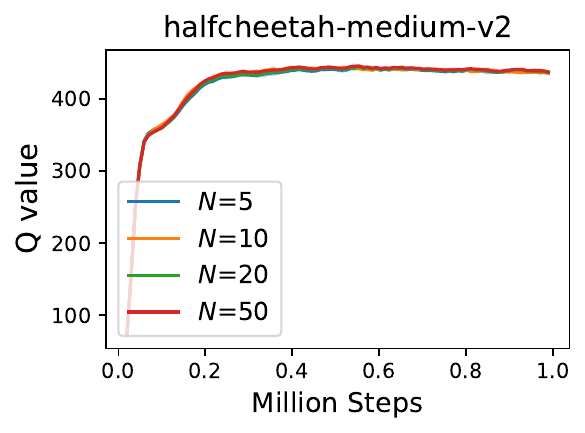}
    }
    \caption{Missing value estimation on halfcheetah-medium-v2.}
    \label{fig:valueestimationhalfcheetah}
\end{figure}

Furthermore, we include more evidence that MCQ induces a good and stable value estimation in Figure \ref{fig:valueestimation} across more datasets with the parameters reported in Section \ref{sec:implementation}. We estimate the in-distribution $Q$-value by $\mathbb{E}_{s,a\sim\mathcal{D}}\mathbb{E}_{i=1,2}\left[Q_{\theta_i}(s,a)\right]$, and estimate the $Q$-value for the OOD actions via $\mathbb{E}_{s\sim\mathcal{D},a\sim\pi(\cdot|s)}\mathbb{E}_{i=1,2}\left[Q_{\theta_i}(s,a)\right]$, and the target value estimates for the in-distribution samples are give by the standard way of calculating target value in SAC. We find that on all of the datasets, MCQ has a good value estimation upon both in-distribution samples and OOD samples. The $Q$ estimation curve is quite smooth and stable on most of them. Based on these observations, we can conclude that MCQ guarantees a good value estimation and will not incur severe overestimation.

\begin{figure}
    \centering
    \subfigure[In-distribution $Q$ value]{
    \label{fig:harandomqpred}
    \includegraphics[width=0.3\textwidth]{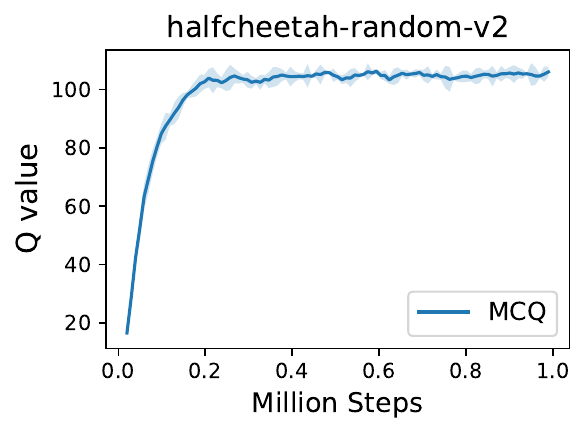}
    }
    \subfigure[OOD $Q$ value]{
    \label{fig:harandomqood}
    \includegraphics[width=0.3\textwidth]{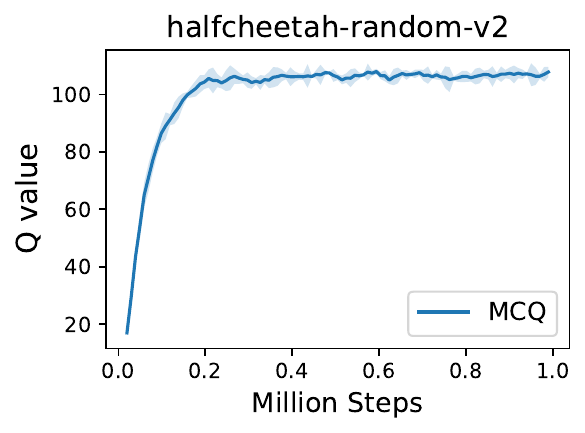}
    }
    \subfigure[In distribution target $Q$ value]{
    \label{fig:harandomqtarget}
    \includegraphics[width=0.3\textwidth]{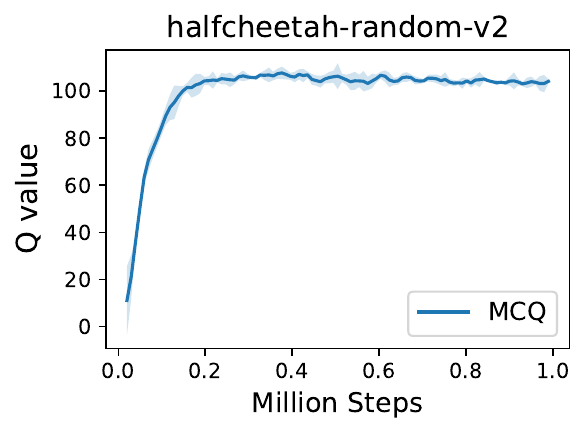}
    }
    \subfigure[In-distribution $Q$ value]{
    \label{fig:warandomqpred}
    \includegraphics[width=0.3\textwidth]{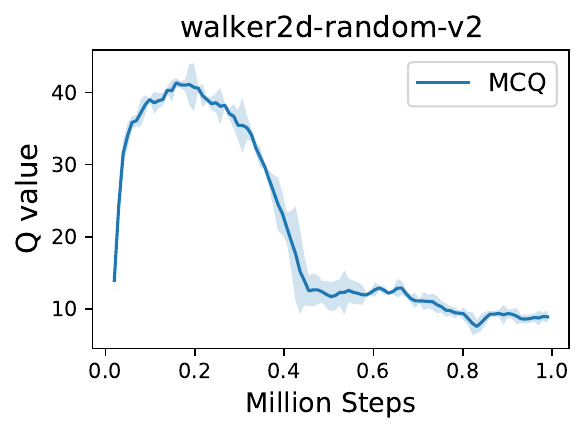}
    }
    \subfigure[OOD $Q$ value]{
    \label{fig:warandomqood}
    \includegraphics[width=0.3\textwidth]{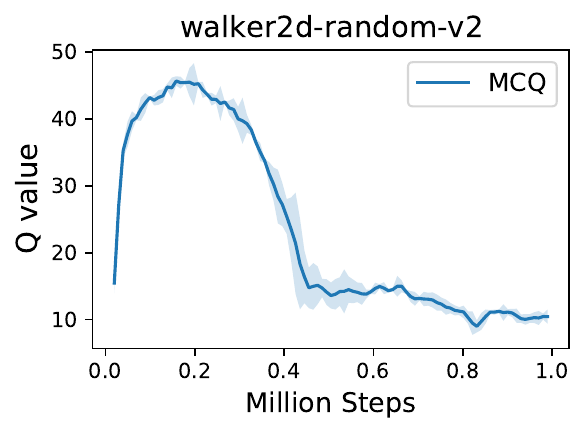}
    }
    \subfigure[In distribution target $Q$ value]{
    \label{fig:warandomqtarget}
    \includegraphics[width=0.3\textwidth]{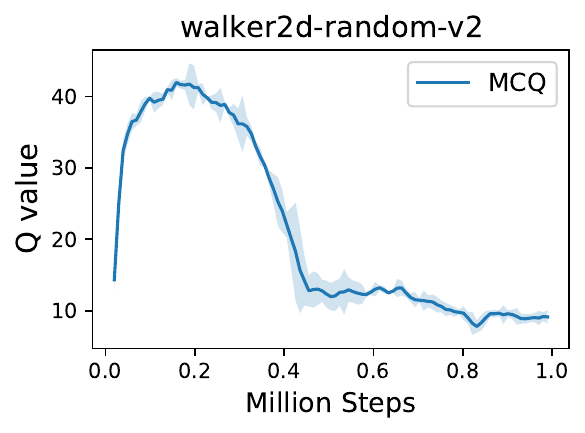}
    }
    \subfigure[In-distribution $Q$ value]{
    \label{fig:hamediumexpertqpred}
    \includegraphics[width=0.3\textwidth]{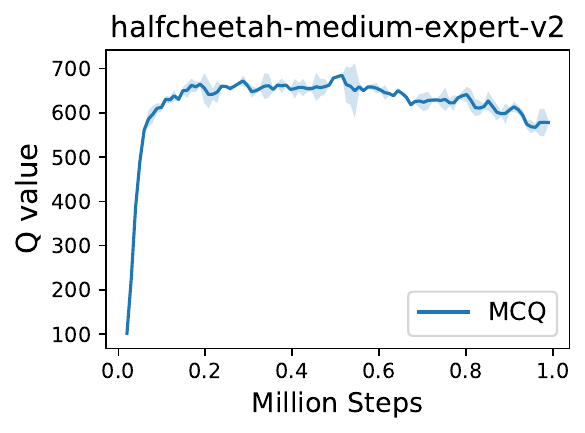}
    }
    \subfigure[OOD $Q$ value]{
    \label{fig:hamediumexpertqood}
    \includegraphics[width=0.3\textwidth]{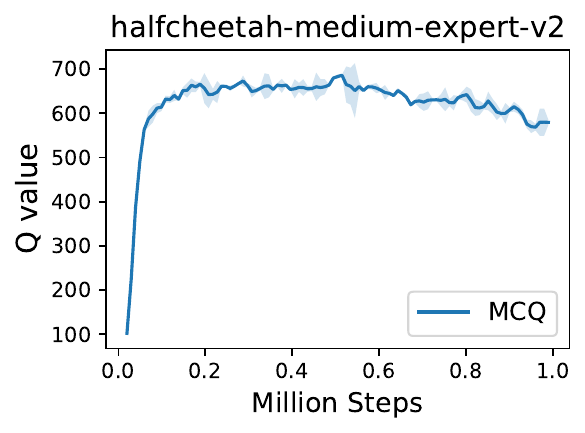}
    }
    \subfigure[In distribution target $Q$ value]{
    \label{fig:hamediumexpertqtarget}
    \includegraphics[width=0.3\textwidth]{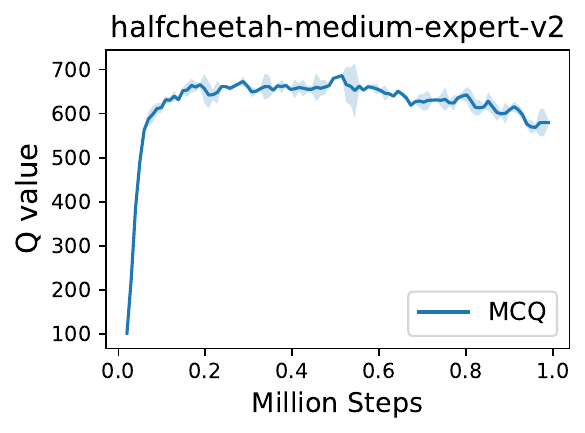}
    }
    \subfigure[In-distribution $Q$ value]{
    \label{fig:wamediumexpertqpred}
    \includegraphics[width=0.3\textwidth]{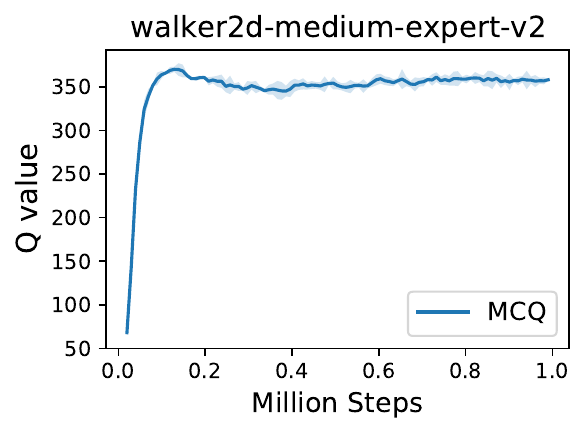}
    }
    \subfigure[OOD $Q$ value]{
    \label{fig:wamediumexpertqood}
    \includegraphics[width=0.3\textwidth]{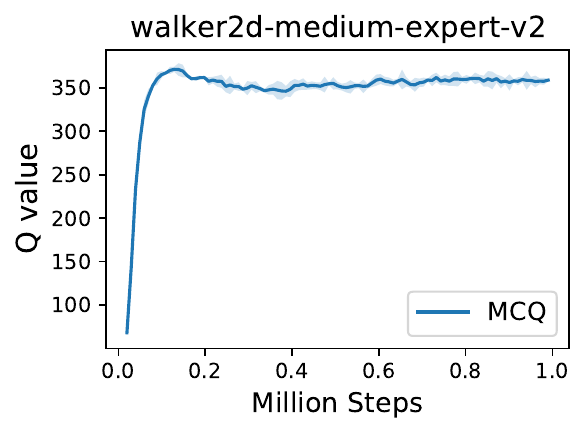}
    }
    \subfigure[In distribution target $Q$ value]{
    \label{fig:wamediumexpertqtarget}
    \includegraphics[width=0.3\textwidth]{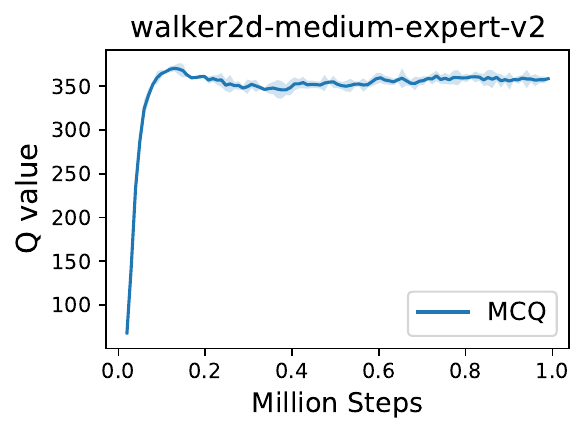}
    }
    \subfigure[In-distribution $Q$ value]{
    \label{fig:haexpertqpred}
    \includegraphics[width=0.3\textwidth]{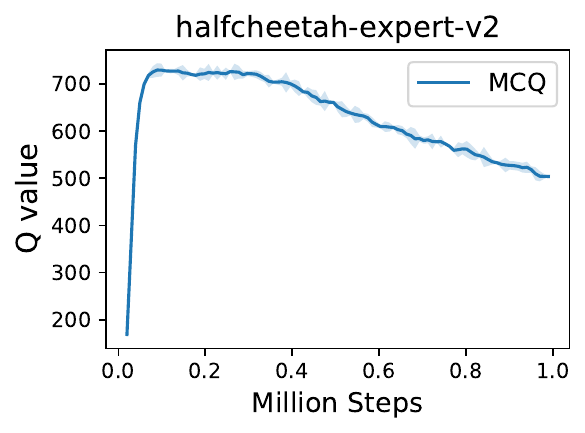}
    }
    \subfigure[OOD $Q$ value]{
    \label{fig:haexpertqood}
    \includegraphics[width=0.3\textwidth]{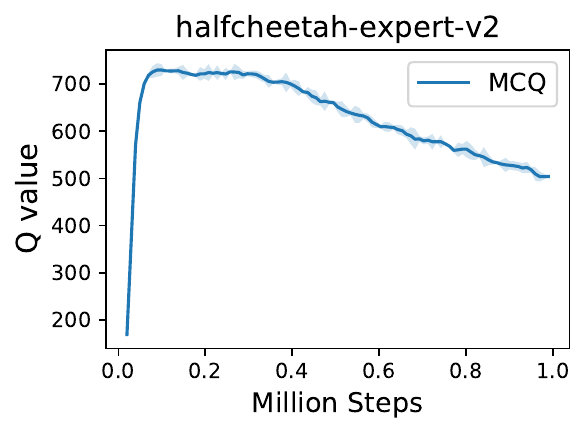}
    }
    \subfigure[In distribution target $Q$ value]{
    \label{fig:haexpertqtarget}
    \includegraphics[width=0.3\textwidth]{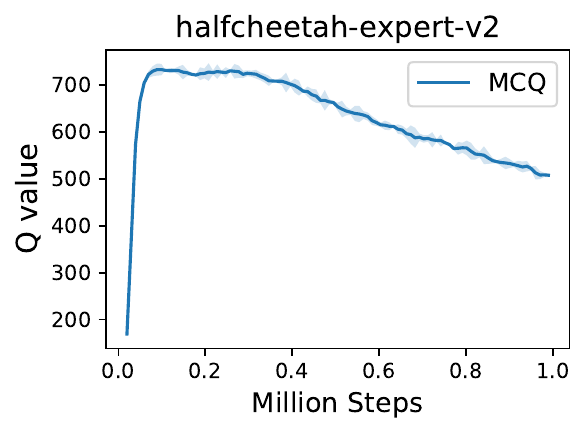}
    }
    \caption{$Q$ value estimation of MCQ on different types of datasets, where $Q$ estimates upon in-dataset samples and OOD samples, target $Q$ value for the in-dataset samples are depicted. The shaded region captures the standard deviation.}
    \label{fig:valueestimation}
\end{figure}

\section{Computation Cost and Compute Infrastructure}

As for the computational cost, MCQ generally consumes 8 to 12 hours to train on MuJoCo tasks due to the cost of OOD sampling. In contrast, CQL takes around 12 to 15 hours. We compare MCQ against CQL on halfcheetah-medium-v2 task in terms of runtime per epoch (1K gradient steps) and GPU memory during training. We conduct the experiments on a single RTX3090 GPU. We detail the computation cost comparison of MCQ against CQL in Table \ref{tab:computationcost}. As shown, MCQ is computationally more efficient than CQL and can achieve better performance with lower computational costs.

\begin{table}[htb]
\caption{Computation cost comparison of MCQ and CQL.}
\label{tab:computationcost}
\centering
\begin{tabular}{c|c|c}
\toprule
\textbf{Algorithm}  & \textbf{Runtime (s/epoch)} & \textbf{GPU Memory (GB)} \\
\midrule
CQL & 38.2 & 1.4 \\
MCQ & 30.3 & 1.2 \\
\bottomrule
\end{tabular}
\end{table}

In Table \ref{tab:computing}, we list the compute infrastructure that we use to run all of the baseline algorithms and MCQ experiments.

\begin{table}[htb]
\caption{Compute infrastructure.}
\label{tab:computing}
\centering
\begin{tabular}{c|c|c}
\toprule
\textbf{CPU}  & \textbf{GPU} & \textbf{Memory} \\
\midrule
AMD EPYC 7452  & RTX3090$\times$8 & 288GB \\
\bottomrule
\end{tabular}
\end{table}

\section{Omitted Background for VAE}
In this section, we give a brief introduction to the VAE, and CVAE \cite{Kingma2014AutoEncodingVB,Sohn2015LearningSO}. Given a dataset $\{x_i\}_{i=1}^N$, the VAE learns a generative model $p_\theta(X) = \prod_{i=1}^N p_\theta(x_i)$, where $X=\{x_i\}_{i=1}^N$ and $\theta$ is the parameter of the generative model. The goal of the VAE is to generate samples that come from the identical distribution as the samples in the dataset, which is equivalent to maximize $p_\theta(x_i)$ for all $x_i$. VAE achieves this by introducing a latent variable $z$ with a prior distribution $p(z)$ and modeling a decoder $p_\theta(X|z)$. The true posterior distribution $p_\theta(z|X)$ is approximated by training an encoder $q_\phi(z|X)$ parameterized by $\phi$. VAE then optimizes the following variational lower bound:
\begin{equation}
    \label{eq:vae}
    \log p_\theta (X) \ge \mathbb{E}_{z\sim q_\phi(\cdot|X)}\left[\log p_\theta(X|z)\right] - \mathrm{KL}(q_\phi(z|X),p(z)).
\end{equation}

Note that $\log p_\theta(X|z)$ denotes the reconstruction loss, and $\mathrm{KL}(p(\cdot),q(\cdot))$ represents the KL-divergence between two probability distribution $p(\cdot)$ and $q(\cdot)$. The prior distribution $p(z)$ is usually set to be a multivariate Gaussian distribution.

Conditional variational auto-encoder (CVAE) is a variant of the vanilla VAE, which aims to model $p_\theta(X|Y)$. The corresponding variational lower bound is given below.
\begin{equation}
    \label{eq:cvae}
    \log p_\theta (X|Y) \ge \mathbb{E}_{z\sim q_\phi(\cdot|X,Y)}\left[\log p_\theta(X|z,Y)\right] - \mathrm{KL}(q_\phi(z|X,Y),p(z|Y)).
\end{equation}

When sampling from the CVAE, we first sample from the prior distribution $p(z|Y)$ and pass it to the decoder to get the desired sample.

\section{Experimental Results on Other Datasets}

To further illustrate the effectiveness of the MCQ algorithm, we conduct more experiments on 4 maze2d "-v1" datasets and 8 Adroit "-v0" datasets from D4RL benchmarks. We compare MCQ against behavior cloning (BC), BEAR \cite{Kumar2019StabilizingOQ}, CQL \cite{Kumar2020ConservativeQF}, BCQ \cite{Fujimoto2019OffPolicyDR}, TD3+BC \cite{Fujimoto2021AMA}, and IQL \cite{kostrikov2022offline}. We take the results of BC, BEAR, CQL, BCQ on these datasets from \cite{Wang2021OfflineRL} directly. For the results of TD3+BC and IQL, we run them on these datasets with their official codebases. All methods are run over 4 different random seeds. The hyperparameters we adopt for these tasks are shown in Table \ref{tab:otherdatasets}. We observe that MCQ behaves fairly well on maze2d datasets, and is competitive to prior methods on Adroit tasks. BC and CQL behaves well on Adroit datasets. Note that we find that BC beats MCQ on hammer-human-v0 and pen-cloned-v0. The reason may be attributed to the high complexity and narrow distribution of the Adroit tasks which makes it hard for $Q$ networks to learn well, and the sampled action can easily lie outside of the dataset's support. Nevertheless, MCQ achieves the highest average score on these datasets.

\begin{table}[htb]
  \caption{Normalized score comparison of different baseline methods on D4RL benchmarks. 0 corresponds to a random policy and 100 corresponds to an expert policy.}
  \label{tab:mazeadroit}
  \small
  \centering
  \begin{tabular}{lllllllllll}
    \toprule
    Task Name & BC & BEAR & CQL & BCQ & TD3+BC & IQL & MCQ (ours) \\
    \midrule
    maze2d-umaze & -3.2 & 65.7 & 18.9 & 49.1 & 25.7$\pm$6.1 & 65.3$\pm$13.4 & \colorbox{yellow}{81.5}$\pm$23.7 \\
    maze2d-umaze-dense & -6.9 & 32.6 & 14.4 & 48.4 & 39.7$\pm$3.8 & 57.8$\pm$12.5 & \colorbox{yellow}{107.8}$\pm$3.2 \\
    maze2d-medium & -0.5 & 25.0 & 14.6 & 17.1 & 19.5$\pm$4.2 & 23.5$\pm$11.1 & \colorbox{yellow}{106.8}$\pm$38.4 \\
    maze2d-medium-dense & 2.7 & 19.1 & 30.5 & 41.1 & 54.9$\pm$6.4 & 28.1$\pm$16.8 & \colorbox{yellow}{112.7}$\pm$5.5\\
    \midrule
    Average Above & -2.0 & 35.6 & 19.6 & 38.9 & 35.0 & 37.2 & \colorbox{yellow}{102.2} \\
    \midrule
    pen-human & 34.4 & -1.0 & 37.5 &  \colorbox{yellow}{68.9} & 0.0$\pm$0.0 & \colorbox{yellow}{68.7}$\pm$8.6 & \colorbox{yellow}{68.5}$\pm$6.5 \\
    door-human & 0.5 & -0.3 & \colorbox{yellow}{9.9} & 0.0 & 0.0$\pm$0.0 & 3.3$\pm$1.3 & 2.3$\pm$2.2 \\
    relocate-human & 0.0 & -0.3 & \colorbox{yellow}{0.2} & -0.1 & 0.0$\pm$0.0 & 0.0$\pm$0.0 & \colorbox{yellow}{0.1}$\pm$0.1 \\
    hammer-human & 1.5 & 0.3 & \colorbox{yellow}{4.4} & 0.5 & 0.0$\pm$0.0 & 1.4$\pm$0.6 & 0.3$\pm$0.1 \\
    pen-cloned & \colorbox{yellow}{56.9} & 26.5 & 39.2 & \colorbox{yellow}{44.0} & 0.0$\pm$0.0 & 35.3$\pm$7.3 & \colorbox{yellow}{49.4}$\pm$4.3 \\
    door-cloned & -0.1 & -0.1 & 0.4 & 0.0 & 0.0$\pm$0.0 & 0.5$\pm$0.6 & \colorbox{yellow}{1.3}$\pm$0.4 \\
    relocate-cloned & -0.1 & -0.3 & -0.1 & -0.3 & \colorbox{yellow}{0.0}$\pm$0.0 & -0.2$\pm$0.0 & \colorbox{yellow}{0.0}$\pm$0.0 \\
    hammer-cloned & 0.8 & 0.3 & \colorbox{yellow}{2.1} & 0.4 & 0.0$\pm$0.0 & \colorbox{yellow}{1.7}$\pm$1.0 & \colorbox{yellow}{1.4}$\pm$0.5 \\
    \midrule
    Average Total & 7.2 & 13.9 & 14.3 & 22.4 & 11.7 & 23.8 & \colorbox{yellow}{44.3} \\
    \bottomrule
  \end{tabular}
\end{table}

\begin{table}[htb]
  \caption{Detailed hyperparameters used in MCQ, where we conduct experiments on D4RL maze2d "-v1" datasets and Adroit "-v0" datasets.}
  \label{tab:otherdatasets}
  \small
  \centering
  \begin{tabular}{l|c|c}
    \toprule
    Task Name  & weighting coefficient $\lambda$ & number of sampled actions $N$ \\
    \midrule
    maze2d-umaze & 0.7 & 10 \\
    maze2d-umaze-dense & 0.95 & 10 \\
    maze2d-medium & 0.9 & 10 \\
    maze2d-medium-dense & 0.9 & 10 \\
    pen-human & 0.3 & 10 \\
    door-human & 0.7 & 10 \\
    relocate-human & 0.3 & 10 \\
    hammer-human & 0.7 & 10 \\
    pen-cloned & 0.9 & 10 \\
    door-cloned & 0.3 & 10 \\
    relocate-cloned & 0.3 & 10 \\
    hammer-cloned & 0.5 & 10 \\
    \bottomrule
  \end{tabular}
\end{table}

\end{document}